%% file: 2024_camera_ready_neurips.tex
\documentclass{article}
\usepackage[utf8]{inputenc}
\usepackage[T1]{fontenc}

\PassOptionsToPackage{numbers, sort&compress}{natbib}
\usepackage[final]{neurips_2024}

\input{packages}
\input{mysymbol}

\definecolor{blush}{rgb}{0.87, 0.36, 0.51}
\usepackage[textwidth=2cm]{todonotes}

\title{Constrained Sampling with\\Primal-Dual Langevin Monte Carlo}

\author{%
Luiz F.\ O.\ Chamon\\
University of Stuttgart\\
\texttt{luiz.chamon@simtech.uni-stuttgart.de}\\
\And
Mohammad Reza Karimi\\
ETH Z\"{u}rich\\
\texttt{mkarimi@inf.ethz.ch}
\And
Anna Korba\\
CREST, ENSAE, IP Paris\\
\texttt{anna.korba@ensae.fr}
}

\begin{document}
\maketitle

\begin{abstract}
This work considers the problem of sampling from a probability distribution known up to a normalization constant while satisfying a set of statistical constraints specified by the expected values of general nonlinear functions. This problem finds applications in, e.g., Bayesian inference, where it can constrain moments to evaluate counterfactual scenarios or enforce desiderata such as prediction fairness. Methods developed to handle support constraints, such as those based on mirror maps, barriers, and penalties, are not suited for this task. This work therefore relies on gradient descent-ascent dynamics in Wasserstein space to put forward a discrete-time primal-dual Langevin Monte Carlo algorithm~(PD-LMC) that simultaneously constrains the target distribution and samples from it. We analyze the convergence of PD-LMC under standard assumptions on the target distribution and constraints, namely (strong) convexity and log-Sobolev inequalities. To do so, we bring classical optimization arguments for saddle-point algorithms to the geometry of Wasserstein space. We illustrate the relevance and effectiveness of PD-LMC in several applications.
\end{abstract}

\section{Introduction}

Sampling is a fundamental task in statistics, with applications to estimation and decision making, and of growing interest in machine learning~(ML), motivated by the need for uncertainty quantification and its success in generative tasks~\cite{Faulkner22s, Schoot21b, Song19g}. In these settings, the distribution we wish to sample from~(\emph{target distribution}) is often known only up to its normalization constant. This is the case, for instance, of score functions learned from data or posterior distributions of complex Bayesian models. Markov Chain Monte Carlo~(MCMC) algorithms can be used to tackle this problem~\cite{Roberts04g, Robert04m} and Langevin Monte Carlo~(LMC), in particular, has attracted considerable attention due to its simplicity, theoretical grounding, and effectiveness in practice~\cite{Roberts96e, Wibisono18s, Durmus19a, Song19g, Wang22a}.
These sampling algorithms, however, do not naturally incorporate requirements on the samples they generate. Specifically, standard MCMC methods do not enforce restrictions on the target distributions, such as support~(e.g., truncated Gaussian), conditional probabilities~(e.g., fairness), or moments~(e.g., portfolio return) constraints. This limitation is often addressed by post-processing, transforming variables, or by introducing penalties in the target distribution. Though successful in specific settings, these approaches have considerable downsides. Post-processing techniques such as \emph{rejection sampling}~(see, e.g., \cite{Lang07b, Li15e}) may substantially reduce the effective number of samples~(number of samples generated per iteration of the algorithm). Variable transformations based on link functions, projections, or mirror/proximal maps~(see, e.g., \cite{Hsieh18m, Bubeck18s, Salim20p, Ahn21e, Kook22s, Sharrock23l, Noble23u}) only accommodate (deterministic)~support constraints and are not suited for statistical requirements such as robustness or fairness~\cite{Madry18t, Kearns18p, Cotter19o, Chamon23c}. Though modifying the target distribution directly offers more flexibility~(see, e.g., \cite{Gurbuzbalaban22p}), it does not \emph{guarantee} that constraints are satisfied~(Table~\ref{T:litreview}). We refer the reader to~\Cref{X:related_work} for a more detailed literature review.

This paper overcomes these issues by directly tackling the \emph{constrained sampling} problem. Explicitly, it seeks to sample not from a target distribution~$\pi$ on $\mathbb{R}^d$, but from the distribution~$\mu^\star$ that solves
\begin{prob}\label{P:constrained_sampling}
    P^\star \triangleq \min_{\mu \in \spprobmeas}& &&\KL(\mu \| \pi)
    \\
    \text{subject to}& &&\E_{x\sim \mu} \!\big[ g_i(x) \big] \leq 0
    	\text{,}\ \ i = 1,\dots,I \text{,}
    \\
    & &&\E_{x\sim \mu} \!\big[ h_j(x) \big] = 0
    	\text{,}\ \ j = 1,\dots,J
    \text{,}
\end{prob}
where~$\spprobmeas$ denotes the set of probability measures on~$\setR^d$ with bounded second moments and the functions~$g_i,h_j$ represent the requirements. Note that~\eqref{P:constrained_sampling} only considers measures~$\mu$ against which~$g_i,h_j$ are integrable. Otherwise, the expectations are taken to be~$+\infty$, making the corresponding measure infeasible.
Observe  that~\eqref{P:constrained_sampling} is more general than the support-constrained sampling problem considered in, e.g., \cite{Hsieh18m, Bubeck18s, Salim20p, Ahn21e, Kook22s, Sharrock23l, Noble23u}. Indeed, it constrains the distribution~$\mu$ rather than its samples~$x$~(\Cref{T:litreview}). Algorithms based on projections, barrier, or mirror maps are not suited for this type of constraints~(see~\Cref{S:csmp} for more details). To tackle~\eqref{P:constrained_sampling}, this paper instead derives and analyzes a primal-dual LMC algorithm~(\pdlmc) that is the sampling counterpart of gradient descent-ascent~(GDA) methods from~(Euclidean) optimization.
A dual ascent algorithm was previously proposed to tackle~\eqref{P:constrained_sampling}, but it requires the exact computation of expectations with respect to intractable distributions~\cite{Liu21s}. This paper not only overcomes this limitation, but also provides convergence guarantees for a broader class of constraint functions.

The main contributions of this work include:
\begin{itemize}
	\item a discrete-time constrained sampling algorithm~(\pdlmc, \Cref{L:pdlmc}) for solving~\eqref{P:constrained_sampling} that precludes any explicit integration~(\Cref{S:primal-dual});

	\item an analysis of \pdlmc proving that it converges sublinearly~(in expectation) with respect to the Kullback-Leibler divergence~(convex case) or Wasserstein distance~(strongly convex case). The analysis is performed directly on the discrete-time iterations and requires only local Lipschitz continuity and bounded variance assumptions~(\Cref{S:convex_convergence});

	\item an extension of these result for target distributions satisfying a log-Sobolev inequality~(LSI) for a variant of \pdlmc~(\Cref{L:dlmc}, \Cref{S:lsi_convergence});

	\item numerical examples illustrating applications of~\eqref{P:constrained_sampling} and the effectiveness of~\pdlmc~(\Cref{S:experiments}).
\end{itemize}

\begin{table}[t]
\centering
\setcellgapes{3pt}\makegapedcells
\caption{Type and target of sampling constraints.}\label{T:litreview}
\begin{tabular}{ccc}
\hline
 & Soft constraint & Hard constraint \\
\hline
Sample~($x\sim\mu$) & --- & \makecell{Mirror/proximal LMC~\cite{Hsieh18m, Salim20p, Ahn21e, Sharrock23l},\\projected LMC~\cite{Bubeck18s},
barriers~\cite{Kook22s, Noble23u}}
\\
Distribution~($\mu$) & Penalized LMC~\cite{Gurbuzbalaban22p} & \textbf{\pdlmc} \\
\hline
\end{tabular}
\end{table}

\section{Problem formulation}

\subsection{Background on Langevin Monte Carlo}

Consider a target distribution~$\pi\in \spprobmeas$  absolutely continuous with respect to Lebesgue measure whose density~(also denoted~$\pi$) can be expressed as~$\pi(x) = e^{-f(x)}/Z$ for some normalization constant~$Z$.
Define the Kullback-Leibler (KL) divergence of~$\mu$  with respect to~$\pi$ as
\begin{equation}\label{E:kl}
	\KL(\mu \| \pi) = \int \log\left(\dfrac{d\mu}{d\pi}\right)d\mu
		 = \int f d\mu + \int \log(\mu) d\mu -\log(Z)
		 \triangleq \calV(\mu) + \calH(\mu) -\log(Z)
		\text{,}
\end{equation}
where~$\frac{d\mu}{d\pi}$ is the Radon--Nikodym derivative, $\calV$ is the \emph{potential energy}, and~$\calH$ is the \emph{negative entropy} if $\mu$ is absolutely continuous with respect to $\pi$; and $+\infty$ otherwise. For a wide class of functions~$f$~(e.g., smooth and strongly convex), samples from~$\pi$ can be obtained from the path of the \emph{Langevin diffusion} process, whose instantaneous values~$x(t)$ have distributions~$\mu(t)$ evolving according to the \emph{Fokker--Planck equation}~\cite{Jordan98t}. Explicitly,
\begin{equation}\label{E:langevin_diffusion}
	d x(t) = - \nabla f(x(t)) dt + \sqrt{2} dW(t)
	\quad \text{and} \quad
	\dfrac{\del \mu(t)}{\del t}
		= \nabla \cdot \big[ \mu(t) \nabla_{W_2}\! \KL(\mu(t) \| \pi) \big]
		\text{,}
\end{equation}
for a~$d$-dimensional Brownian motion~$W(t)$, where~$\nabla \cdot q$ denotes the divergence of~$q$ and $\nabla_{W_2}\KL(\mu \| \pi)$ denotes the Wasserstein-2 gradient of~$\KL(\cdot|\pi)$ at~$\mu$ \cite[Theorem 10.4.17]{Ambrosio05g} (see \Cref{X:background_wass} for more details). Indeed, the Langevin diffusion~\eqref{E:langevin_diffusion} brings the distribution~$\mu(t)$ of~$x(t)$ progressively closer to the target~$\pi$. In fact, the Fokker-Planck equation can be interpreted as a gradient flow of the KL divergence with respect to the Wasserstein-2 distance~\cite{Jordan98t, Wibisono18s}.

However, computing the path of the stochastic differential equation in~\eqref{E:langevin_diffusion} is not practical and discretizations are used instead. Chief among them is the (forward)~Euler--Maruyama scheme, which leads to the celebrated Langevin Monte Carlo (LMC) algorithm~\cite{Roberts96e}
\begin{equation}\label{E:lmc}
	x_{k+1} = x_k - \gamma_k \nabla f(x_k) + \sqrt{2 \gamma_k} \beta_k
		\text{,} \quad \beta_k \stackrel{\text{iid}}{\sim} \calN(0, \Id)
		\text{,}
\end{equation}
for a step size~$\gamma_k > 0$, where~$\Id$ denotes the~$d$-dimensional identity matrix.
Notice that it is not necessary to know~$Z$ in order to evaluate~\eqref{E:lmc}.
This has made LMC and its variants widely popular in practice and the subject of extensive research. Despite~\eqref{E:lmc}  being a biased time-discretizations of the Langevin diffusion in~\eqref{E:langevin_diffusion}~\cite{Wibisono18s}, rates of convergence of LMC have been obtained for smooth and strongly convex~\cite{Durmus19a,dalalyan2019user} or convex~\cite{dalalyan2022bounding} potentials or when the target distribution~$\pi$ verifies an LSI~\cite{Vempala19r}.

\subsection{Constrained sampling}
\label{S:csmp}

Our goal, however, is not to sample from~$\pi$ itself, but from a distribution close to~$\pi$ that satisfies a set of statistical requirements. Explicitly, we wish to sample from a distribution~$\mu^\star$ that solves~\eqref{P:constrained_sampling}. Since~\eqref{P:constrained_sampling} constrains the distribution~$\mu$ rather than its samples~$x$, it can accommodate more general requirements than the support constraints typically considered in constrained sampling~(e.g., \cite{Hsieh18m, Bubeck18s, Salim20p, Ahn21e, Kook22s, Sharrock23l, Noble23u}). Next, we illustrate the wide range of practical problems that can be formulated as~\eqref{P:constrained_sampling}. These examples are further explored in~\Cref{S:experiments} and more details on their formulations are provided in~\Cref{X:applications}.
\begin{enumerate}
	\item \textbf{Sampling from convex sets}: though we have stressed that~\eqref{P:constrained_sampling} accommodates other types of requirements, it can also be used to constrain the support of~$\pi$, i.e., to sample from
	\begin{prob}\label{P:support_sampling}
	    \mu^\star \in \argmin_{\mu \in \spprobmeas}& &&\KL(\mu \| \pi)
	    \\
	    \subjectto& &&\Pr_{x \sim \mu} [x \in \calC] = 1
	    	\text{,}
	\end{prob}
	for a closed convex set~$\calC \subset \setR^d$. Indeed, let~$\calC$ be the intersection of the~$0$-sublevel sets of convex functions~$\{s_i\}_{i=1,\dots,I}$. Such a description always exists~(see~\Cref{X:applications}). Then, \eqref{P:support_sampling} can be cast as~\eqref{P:constrained_sampling} using~$g_i(x) = [ s_i(x) ]_+$ for~$[ z ]_+ = \max(0,z)$. Notice that the~$g_i$ are convex and that although they are not everywhere differentiable, $\indicator(s_i(x) > 0) \nabla s_i(x)$ is a \emph{subgradient} of~$g_i$, where~$\indicator(\calE) = 1$ on the event~$\calE$ and~$0$ otherwise. Observe that support constraints can also be imposed using projections, mirror/proximal maps, and barriers as in~\cite{Hsieh18m, Bubeck18s, Salim20p, Ahn21e, Kook22s, Sharrock23l, Noble23u}. These methods, however, constrain the samples~$x$ rather than their distribution~$\mu$ as in~\eqref{P:support_sampling}.

	\item \textbf{Rate-constrained Bayesian models}: rate constraints have garnered attention in ML due to their central role in fairness~\cite{Kearns18p, Cotter19o}. Consider data pairs~$(x,y)$, where~$x \in \calX$ are features and~$y \in \{0,1\}$ labels, and a protected~(measurable) subgroup~$\calG \subset \calX$. Let~$\pi$ be a Bayesian posterior of the parameters~$\theta$ of a model~$q(\cdot;\theta)$ denoting the probability of a positive outcome~(based, e.g., on a binomial model). We wish to enforce statistical parity, i.e., we wish the prevalence of positive outcomes within the protected group~$\calG$ to be close to or higher than in the whole population. We cast this problem as~\eqref{P:constrained_sampling} by constraining the average probability of positive outcome as in
	\begin{prob}\label{P:fairness}
		P^\star = \min_{\mu \in \spprobmeas}& &&\KL(\mu \| \pi)
		\\
		\subjectto& &&\E_{x, \theta \sim \mu}
		\!\big[ q(x ; \theta) \mid \calG \big] \geq
		\E_{x, \theta \sim \mu}
		\!\big[ q(x ; \theta) \big] - \delta
		\text{.}
	\end{prob}
	where~$\delta > 0$ denotes our tolerance~\cite{Chamon23c}. Naturally, multiple protected groups can be accommodated by incorporating additional constraints. Hence, constrained sampling provides a natural way to encode fairness in Bayesian inference.

	\item \textbf{Counterfactual sampling}: rather than imposing requirements on probabilistic models, constrained sampling can also be used to probe them by evaluating \emph{counterfactual} statements. Indeed, let~$\pi$ denote a reference probabilistic model such that sampling from~$\pi$ yields realizations of the ``real world.'' Consider the \emph{counterfactual} statement ``how would the world have been if~$\E[g(x)] \leq 0$?'' Constrained sampling not only gives realizations of this alternative world, but it also indicates its ``compatibility'' with the reference model, namely the value~$P^\star$ of~\eqref{P:constrained_sampling}.

	More concretely, consider a \emph{Bayesian stock market} model. Here, $\pi$ is a posterior model for the (log-)returns of~$I$ assets, e.g., distributed as  Gaussians~$\calN(\rho,\Sigma)$. Here, the vector $\rho$ describes the mean return of each stock and~$\Sigma$ their covariance. We can investigate what the market would look like if, e.g., the mean and variance of each stocks were to change by solving
	\begin{prob}\label{P:porfolio}
	    P^\star = \min_{\mu \in \spprobmeas}& &&\KL(\mu \| \pi)
	    \\
	    \subjectto&
	    	&&\E_{(\rho,\Sigma) \sim \mu} \!\big[ \rho_i \big] = \bar{\rho}_i
	    	\text{,} \quad i = 1,\dots,I
	    \\
	    & &&\E_{(\rho,\Sigma) \sim \mu} \!\big[ \Sigma_{ii} \big] \leq \bar{\sigma}_i^2
	\end{prob}
	Due to correlations in the market, certain choices of~$\bar{\rho}_i$ or~$\bar{\sigma}_i^2$ may be more ``unrealistic'' than others. Additionally, it could be that some of these conditions are vacuous conditioned on the others. As we show next, our approach to tackling~\eqref{P:constrained_sampling} effectively isolates the contribution of each requirement in the solution~$\mu^\star$, thus enabling us to identify which are (conditionally)~vacuous and which are most at odds with the reference model~$\pi$.
\end{enumerate}

\subsection{Lagrangian duality and dual ascent algorithms}
\label{S:lagrangian}

Although directly sampling from~$\mu^\star$ does not appear straightforward, it admits a convenient characterization based on convex duality that is amenable to be sampled using the LMC algorithm~\eqref{E:lmc}. Indeed, let~$g: \setR^d \to \setR^I$ and~$h: \setR^d \to \setR^J$ be vector-valued functions  collecting the constraint functions~$g_i$ and~$h_j$ respectively. The Lagrangian of~\eqref{P:constrained_sampling} is then defined as
\begin{equation}\label{E:lagrangian}
    L(\mu, \lambda, \nu) \triangleq \KL(\mu\|\pi)
    	+ \lambda^\top \E_{\mu}[ g ]
    	+ \nu^\top \E_{\mu}[ h ]
    = \KL(\mu\|\mu_{\lambda\nu}) + \log\left( \frac{Z}{Z_{\lambda\nu}} \right)
  		\text{,}
\end{equation}
for~$\lambda \in \setR^I_+$ and~$\nu \in \setR^J$, where
\begin{equation}\label{E:U_potential}
    \mu_{\lambda\nu}(x) = \dfrac{e^{-U(x,\lambda,\nu)}}{Z_{\lambda\nu}} \quad \text{for} \quad
    	U(x,\lambda,\nu) = f(x) + \lambda^\top g(x) + \nu^\top h(x)
\end{equation}
and a normalization constant~$Z_{\lambda\nu}$. Notice that~$P^\star = \min_{\mu} \max_{\lambda \geq 0,\,\nu} L(\mu,\lambda,\nu)$, which is why~\eqref{P:constrained_sampling} is referred to as the \emph{primal problem}.

To obtain the \emph{dual problem} of~\eqref{P:constrained_sampling}, define the dual function
\begin{equation}\label{E:dual_function}
    d(\lambda,\nu) \triangleq \min_{\mu \in \spprobmeas}\ L(\mu, \lambda, \nu)
    	\text{.}
\end{equation}
Notice from~\eqref{E:lagrangian} that the minimum in~\eqref{E:dual_function} is achieved for~$\mu_{\lambda\nu}$ from~\eqref{E:U_potential}, the \emph{Lagrangian minimizer}, so that~$d(\lambda,\nu) = \log(Z / Z_{\lambda\nu})$. The solution of~\eqref{E:dual_function} is therefore a \emph{tilted} version of~$\pi$, whose tilt is controlled by the \emph{dual variables}~$(\lambda,\nu)$. Since~\eqref{E:dual_function} is a relaxation of~\eqref{P:constrained_sampling}, it yields a lower bound on the primal value, i.e., $d(\lambda,\nu) \leq P^\star$ for all~$(\lambda,\nu) \in \setR^I_+ \times \setR^J$. The dual problem seeks the tilts~$(\lambda^\star,\nu^\star)$ that yield the best lower bound, i.e.,
\begin{prob}[\textup{DI}]\label{P:dual}
    D^\star \triangleq \max_{\lambda \in \setR^I_+,\,\nu \in \setR^J}\ d(\lambda,\nu)
        \text{.}
\end{prob}
The set~$\Phi^\star = \argmax_{\lambda\geq 0,\,\nu}\ d(\lambda,\nu)$ of solutions of~\eqref{P:dual} is called the set of \emph{Lagrange multipliers}. Note from~\eqref{E:dual_function} that~\eqref{P:dual} depends on the distributions~$\mu$ and~$\pi$ through its objective~$d$.

The dual problem~\eqref{P:dual} has several advantageous properties. Indeed, while the primal problem~\eqref{P:constrained_sampling} is an \emph{infinite dimensional, smooth} optimization problem in probability space, the dual problem~\eqref{P:dual} is a \emph{finite dimensional, non-smooth} optimization problem in Euclidean space. What is more, it is a concave problem regardless of the functions~$f,g,h$, since the dual function~\eqref{E:dual_function} is the minimum of a set of affine functions in~$(\lambda,\nu)$~\cite[Prop.~4.1.1]{Bertsekas09c}. These properties are all the more attractive given that, under mild conditions stated below, \eqref{P:dual} can be used to solve~\eqref{P:constrained_sampling}.

\begin{assumption}\label{A:strict_feasibility}
	There exists~$\mu^\dagger \in \spprobmeas$ with~$\KL(\mu^\dagger \| \pi) \leq C < \infty$ such that~$\E_{\mu^\dagger}[g_i] \leq -\delta < 0$ and~$\E_{\mu^\dagger}[h_j] = 0$ for all~$i,j$.
\end{assumption}

\begin{proposition}\label{T:strong_duality}
	Under Assumption~\ref{A:strict_feasibility}, the following holds:
\begin{enumerate}[label=(\roman*), itemsep=-1pt, leftmargin=20pt, labelsep=5pt, topsep=0pt]
    \item ~$P^\star = D^\star$;
    \item there exists a finite pair~$(\lambda^\star,\nu^\star) \in \Phi^\star$;
    \item for any solution~$\mu^\star$ of~\eqref{P:constrained_sampling} and~$(\lambda^\star,\nu^\star)$ of~\eqref{P:dual}, it holds that
    \begin{equation}\label{E:saddle_point}
      L(\mu^\star,\lambda,\nu)
          \leq L(\mu^\star,\lambda^\star,\nu^\star)
          \leq L(\mu,\lambda^\star,\nu^\star)
          \text{,} \quad \text{for all $(\mu,\lambda,\nu) \in \spprobmeas \times \setR^I_+ \times \setR^J$;}
    \end{equation}
    \item the solution of~\eqref{P:constrained_sampling} is~$\mu^\star = \mu_{\lambda^\star\nu^\star}$ for~$(\lambda^\star, \nu^\star) \in \Phi^\star$;
    \item consider the perturbation of~\eqref{P:constrained_sampling}
    \begin{prob}\label{P:perturbed}
      P^\star(u,v) \triangleq \min_{\mu \in \spprobmeas}\ \KL(\mu \| \pi)\ \ %
      \subjectto\ \ \E_{x\sim \mu} \!\big[ g_i(x) \big] \leq u_i
          \text{,}\ \E_{x\sim \mu} \!\big[ h_j(x) \big] = v_j
          \text{.}
    \end{prob}
    Then, $(\lambda^\star,\nu^\star)$ are subgradients of~$P^\star(0,0) = P^\star$, i.e., $P^\star(u,v) \geq P^\star - {\lambda^\star}^\top u - {\nu^\star}^\top v$, and if~$P^\star(u,v)$ is differentiable at~$(0,0)$, then~$\nabla_{u} P^\star(u,v) = -\lambda^\star$ and~$\nabla_{v} P^\star(u,v) = -\nu^\star$ at~$(0,0)$.
\end{enumerate}

\end{proposition}

\begin{proof}
In finite dimensional settings, (i)--(v) are well-known duality results~(see, e.g., \cite{Bertsekas09c}). While they also hold for infinite dimensional optimization problems, their proofs are slightly more ``scattered.'' We collect their reference below.
The objective of~\eqref{P:constrained_sampling} is a convex function and its constraints are linear functions of~$\mu$. Hence, \eqref{P:constrained_sampling} is a convex program.
Under Slater's condition~(Assumption~\ref{A:strict_feasibility}), it is (i)~strongly dual~($P^\star = D^\star$) and (ii)~there exists at least one solution~$(\lambda^\star,\nu^\star)$ of~\eqref{P:dual}~(see~\cite[Sec.~8.6, Thm.~1]{Luenberger68o} or~\cite[Cor.~4.1]{Jeyakumar92g}).
This implies (iii)~the existence of the saddle-point~\eqref{E:saddle_point}~\cite[Prop.~2.156]{Bonnans00p},
(iv)~that~$\mu^\star \in \argmin_\mu L(\mu,\lambda^\star,\nu^\star) = \{\mu_{\lambda^\star\nu^\star}\}$, since the KL divergence is strongly convex and its minimizer is unique~\cite[Thm.~7.3.7]{Kurdila05c}, and
(v)~that~$(\lambda^\star,\nu^\star)$ are subgradients of the perturbation function~$P^\star(u,v)$~\cite[Prop.~4.27]{Bonnans00p}.
\end{proof}

\Cref{T:strong_duality} shows that given solutions~$(\lambda^\star,\nu^\star)$ of~\eqref{P:dual}, the constrained sampling problem~\eqref{P:constrained_sampling} reduces to sampling from~$\mu_{\lambda^\star\nu^\star} \propto e^{-U(\cdot,\lambda^\star,\nu^\star)}$~(see~\Cref{X:gaussian_duality} for an explicit example of this result). It is important to note that this results only relies on the KL divergence being~(strongly) convex in the standard~$L^2$ geometry, i.e., along mixtures of the form~$t \mu_0 + (1-t) \mu_1$ for $t \in [0,1]$. This does not imply that it is (geodesically)~convex in the Wasserstein sense~\cite[Section 9.1.2]{Villani21t}. This would require~$U$ in~\eqref{E:U_potential} to be convex for all~$\lambda \geq 0$ and~$\nu \in \setR^J$, i.e., for~$f,g$ to be convex and~$h$ to be linear.

Hence, \Cref{T:strong_duality} reduces the constrained sampling problem~\eqref{P:constrained_sampling} to that of finding the Lagrange multipliers~$(\lambda^\star,\nu^\star)$. Despite their finite dimensionality, however, computing these parameters is intricate. Indeed, since~\eqref{P:dual} is a concave program, we could obtain~$(\lambda^\star,\nu^\star) \in \Phi^\star$ using
\begin{equation}\label{E:dual_ascent}
	\lambda_{k+1} = \big[ \lambda_{k} + \eta_k \E_{\mu_{\lambda_k\nu_k}} [g] \big]_+
	\quad \text{and} \quad
	\nu_{k+1} = \nu_{k} + \eta_k \E_{\mu_{\lambda_k\nu_k}} [h]
		\text{,}
\end{equation}
for~$\eta_k > 0$, where we used the fact that~$\E_{\mu_{\lambda\nu}}[g]$ and~$\E_{\mu_{\lambda\nu}}[h]$ are (sub)gradients of the dual function~\eqref{E:dual_function} at~$(\lambda,\nu)$~\cite[Thm.~2.87]{Ruszczynski06n}. This procedure is known in optimization and game theory as \emph{dual ascent} or \emph{best response}~\cite{Nisan07a, Bertsekas96c}. Notice, however, that~\eqref{E:dual_ascent} is not a practical algorithm as it requires explicit integration with respect to the intractable distribution~$\mu_{\lambda\nu}$ from~\eqref{E:U_potential}.

This issue was partially addressed in~\cite{Liu21s}~(in continuous time and without equality constraints, i.e., $J = 0$) by replacing the Lagrangian minimizer~$\mu_{\lambda_k\nu_k}$ in~\eqref{E:dual_ascent} by the distribution of LMC samples, as in
\begin{equation}\label{E:simultaneous_dual_ascent}
\begin{aligned}
	x_{k+1} &= x_k - \gamma_k \nabla U(x_k,\lambda_k) + \sqrt{2 \gamma_k} \beta_k
	\text{,}\quad \beta_k \stackrel{\text{iid}}{\sim} \calN(0, \Id)
	\\
	\lambda_{k+1} &= \Big[ \lambda_{k} + \eta_k \E_{\mu_{k}} \![g] \Big]_+
		\text{,} \quad \mu_k = \text{Law}(x_k)
		\text{.}
\end{aligned}
\end{equation}
Note that since~$J = 0$, we omit the argument~$\nu$ of~$U$ for clarity. Nevertheless, the updates in~\eqref{E:simultaneous_dual_ascent} still require an explicit integration. While it is now possible to sample from~$\mu_k$~(namely, using the~$x_k$), empirical approximations of~$\E_{\mu_k} \![g]$ may not only require an exponential~(in the dimension~$d$) number of samples~(e.g.,~\cite[Thm.~1.2]{Kloeckner12a}), but it introduces errors that are not taken into account in the analysis of~\cite{Liu21s}. In the sequel, we address these drawbacks by replacing these dual ascent algorithms by a saddle-point one.

\section{Primal-dual Langevin Monte Carlo}
\label{S:primal-dual}

Consider the GDA dynamics for the saddle-point problem~\eqref{P:dual} in Wasserstein space. Explicitly,
\begin{subequations}\label{E:wasserstein_gda}
\begin{align}
    \dfrac{\del \mu(t)}{\del t}
    	&= \nabla \cdot \big[ \mu(t) \nabla_{W_2} L\big(\mu(t),\lambda(t),\nu(t)\big) \big]
    	\label{E:wasserstein_gda_mu}
    \\
    \dfrac{\del \lambda_{i}(t)}{\del t} &=
    	\big[ \nabla_{\lambda_i} L\big(\mu(t),\lambda(t),\nu(t)\big) \big]_{\lambda_i(t),+}
    	\label{E:wasserstein_gda_lambda}
    \\
    \dfrac{\del \nu_{j}(t)}{\del t} &= \nabla_{\nu_j} L\big(\mu(t),\lambda(t),\nu(t)\big)
    	\label{E:wasserstein_gda_nu}
\end{align}
\end{subequations}
for the Lagrangian~$L$ defined in~\eqref{E:lagrangian}, where~$[z]_{\lambda,+} = z$ for~$\lambda > 0$ and~$[z]_{\lambda,+} = \max(a, 0)$ otherwise~\cite[Sec.~2.2]{Nagurney96p}. Observe that~$\nabla_{\lambda_i} L(\mu, \lambda, \nu) = \mathbb{E}_{\mu} [g_i]$ and~$\nabla_{\nu_j} L(\mu, \lambda, \nu) = \mathbb{E}_{\mu} [h_j]$. Hence, the algorithm from~\cite{Liu21s} described in~\eqref{E:simultaneous_dual_ascent} involves a \emph{deterministic} implementation of~\eqref{E:wasserstein_gda_lambda} that fully integrates over~$\mu(t)$. In contrast, we consider a \emph{stochastic}, \emph{single-particle} implementation of~\eqref{E:wasserstein_gda} that leads to the practical procedure in~\Cref{L:pdlmc}.

\begin{algorithm}[t]
\caption{Primal-dual LMC}
	\label{L:pdlmc}
\begin{algorithmic}[1]
	\State \textbf{Inputs:} $\eta_k > 0$~(step size), $x_{0} \sim \mu_0$, and~$(\lambda_{0},\nu_{0}) = (0,0)$.

	\For{$k = 0,\dots,K-1$}
		\State $\displaystyle x_{k+1} = x_k - \eta_k \nabla_{x} U(x_k, \lambda_k, \nu_k) + \sqrt{2\eta_k}\,\beta_k$, \quad for $\beta_k \sim \calN(0, \Id)$
			\label{E:pdlmc_x}

		\State $\displaystyle \lambda_{k+1} = \big[ \lambda_{k} + \eta_k g(x_k) \big]_+$
			\label{E:pdlmc_lambda}

		\State $\displaystyle \nu_{k+1} = \nu_{k} + \eta_k h(x_k)$
			\label{E:pdlmc_nu}
	\EndFor
\end{algorithmic}
\end{algorithm}

Explicitly, we also use an Euler-Maruyama time-discretization of the Langevin diffusion associated to~\eqref{E:wasserstein_gda_mu}~(step~\ref{E:pdlmc_x}), but replace the expectations in~\eqref{E:wasserstein_gda_lambda}--\eqref{E:wasserstein_gda_nu} by single-sample approximations~(steps~\ref{E:pdlmc_lambda}--\ref{E:pdlmc_nu}). \Cref{L:pdlmc} can therefore be seen as a particle implementation of the deterministic Wasserstein GDA algorithm~\eqref{E:wasserstein_gda}. As such, it resembles a primal-dual counterpart of the LMC algorithm in~\eqref{E:lmc}, which is why we dub it \emph{primal-dual LMC}~(\pdlmc). Alternatively, \Cref{L:pdlmc} can be interpreted as a stochastic approximation of the dual ascent method in~\eqref{E:simultaneous_dual_ascent}. This suggests that the gradient approximations in steps~\ref{E:pdlmc_lambda}--\ref{E:pdlmc_nu} could be improved using mini-batches, which is in fact how~\cite{Liu21s} approximates the expectation in~\eqref{E:simultaneous_dual_ascent}. Our theoretical analysis and experiments show that these mini-batches are neither necessary nor always worth the additional computational cost~(see~\Cref{S:convex_convergence} and~\Cref{S:experiments}). Note that the ``stochastic approximations'' in Algorithm~\ref{L:pdlmc} refer to the dual updates~(steps~\ref{E:pdlmc_lambda}--\ref{E:pdlmc_nu}) rather than the LMC update~(step~\ref{E:pdlmc_x}) as in stochastic gradient Langevin~\cite{Welling11b}. Though these methods could be combined, it is beyond the scope of this work.

The remainder of this section is dedicated to analyzing the convergence properties of \pdlmc for both stochastic dual gradients~(as in Algorithm~\ref{L:pdlmc}) and exact dual gradient~(as in~\eqref{E:simultaneous_dual_ascent}). For the latter, we obtain guarantees for the discrete implementation~\eqref{E:simultaneous_dual_ascent} under weaker assumptions than the continuous-time analysis of~\cite{Liu21s}. We consider strongly log-concave target distributions in Section~\ref{S:convex_convergence} and those satisfying an LSI in~\Cref{S:lsi_convergence}.

\subsection{PD-LMC with~(strongly) convex potentials}
\label{S:convex_convergence}

As opposed to the traditional LMC algorithm~\eqref{E:lmc} or the deterministic updates in~\eqref{E:simultaneous_dual_ascent}, Algorithm~\ref{L:pdlmc} involves three coupled random variables, namely, $(x_k,\lambda_k,\nu_k)$. Hence, the LMC update~(step~\ref{E:pdlmc_x}) is based on a \emph{stochastic} potential~$U$ and the distribution~$\mu_k$ of~$x_k$ is now a \emph{random measure}. Our analysis sidesteps this obstacle by using techniques from stochastic optimization. We also leverage techniques from primal-dual algorithms in the Wasserstein space, in the spirit of works such as~\cite{Durmus19a, Wibisono18s,Salim20t} that studied the LMC~\eqref{E:lmc} or alternative time-discretizations of gradient flows of the KL divergence as splitting schemes.

First, define the potential energy~$\calE$ and the (negative) entropy~$\calH$ for~$(\mu,\lambda,\nu) \in \spprobmeas \times \setR^I_+ \times \setR^J$~as
\begin{equation}\label{E:potential_entropy}
	\calE(\mu,\lambda,\nu) = \int U(x,\lambda,\nu) d\mu(x)
	\quad\text{and}\quad
	\calH(\mu) = \int \log(\mu) d\mu
		\text{,}
\end{equation}
for~$U$ as in~\eqref{E:U_potential}. Notice from~\eqref{E:lagrangian} that~$L(\mu,\lambda,\nu) = \mathcal{E}(\mu, \lambda, \nu) + \mathcal{H}(\mu) - \log(Z)$, where we used the KL divergence decomposition in~\eqref{E:kl}. To proceed, consider the following assumptions:

\begin{assumption}\label{A:convexity}
	The potential energy~$\calE(\mu,\lambda,\nu)$ in~\eqref{E:potential_entropy} is $m$-strongly convex with respect to~$\mu$ along Wasserstein-2 geodesics for~$m \geq 0$ and all~$(\lambda,\nu) \in \setR^I_+ \times \setR^J$. Explicitly,
	\begin{equation*}
	    \calE(\mu,\lambda,\nu) \geq \calE(\mu_0,\lambda,\nu)
		    + \int \langle \nabla_{W_2} \calE(\mu_0,\lambda,\nu), x-y \rangle ds(x,y)
		    + \dfrac{m}{2} W_2^2(\mu,\mu_0)
	    	\text{,}
	\end{equation*}
	where~$s$ is an optimal coupling achieving~$W_2^2(\mu,\mu_0)$~(see \Cref{X:background_wass}).
\end{assumption}

\begin{assumption}\label{A:bounded_gradients}
    The gradients and variances of~$f,g,h$ are bounded along iterations~$\{\mu_k\}_{k \geq 0}$, where~$\mu_k$ is the distribution of~$x_k$, i.e., there exists~$G^2$ such that
    \begin{equation*}
        \max \big( \| \nabla f \|_{L^2(\mu_k)}^2,
        	\| \nabla g_i \|_{L^2(\mu_k)}^2,
        	\| \nabla h_j \|_{L^2(\mu_k)}^2 \big) \leq G^2
        \ \ \text{and} \ \ %
        \max \big( \E_{\mu_k}[\norm{g}^2], \E_{\mu_k}[\norm{h}^2] \big) \leq G^2
        \text{.}
    \end{equation*}
\end{assumption}

Assumption~\ref{A:convexity} holds with~$m = 0$ if~$f,g$ are convex and~$h$ is linear. If~$f$ is additionally strongly convex, then it holds with~$m > 0$~\cite[Prop.~9.3.2]{Ambrosio05g}. Assumption~\ref{A:bounded_gradients} is typical in~(stochastic) non-smooth optimization analyses~(see, e.g., \cite{Nedic09a, cherukuri2016asymptotic, Ruszczynski06n}). Notice, however, that gradients are only required to be bounded along trajectories of Algorithm~\ref{L:pdlmc}, a crucial distinction in the case of strongly convex functions whose gradients can only be bounded locally. Assumption~\ref{A:bounded_gradients} can be satisfied under mild conditions on~$f,g,h$, such as local Lipschitz continuity or linear growth.

The following theorem provides the first convergence analysis of the discrete-time \pdlmc.

\begin{theorem}\label{T:main_convex}
	Denote by~$\mu_k$ the distribution of~$x_k$ in Algorithm~\ref{L:pdlmc}. Under Assumptions~\ref{A:strict_feasibility}, \ref{A:convexity}, and~\ref{A:bounded_gradients}, there exists~$R_0^2$ such that, for~$\eta_k \leq \eta$,
	\begin{equation}\label{E:main_convex_fix}
		\dfrac{1}{K} \sum_{k = 1}^K \Big[
			\KL(\mu_{k} \,\|\, \mu^\star)
			+ \dfrac{m}{2} W_2^2(\mu_k,\mu^\star)
		\Big] \leq
		3\eta G^2 + \dfrac{R_0^2}{\eta K}
		+ \dfrac{\eta G^2}{K}\sum_{k = 1}^K
			\big( \E[\norm{\lambda_k}^2] + \E[\norm{\nu_k}^2] \big)
			\text{.}
	\end{equation}
	For~$\eta_k \leq R_0/(G\sqrt{k})$ and~$\bar{\eta}_k = \eta_k/\sum_{k = 1}^K \eta_k$, we obtain
	\begin{equation}\label{E:main_convex_dec}
		\sum_{k = 1}^K \bar{\eta}_k \Big[
			\KL(\mu_{k} \,\|\, \mu^\star)
			+ \dfrac{m}{2} W_2^2(\mu_k,\mu^\star)
		\Big]\leq
		\dfrac{R_0 G (1 + \log(K))}{\sqrt{K}} \Big(
			3 + \max_k\ \big\{ \E[ \norm{\lambda_k}^2 ] + \E[ \norm{\nu_k}^2 ] \big\}
		\Big)
			\text{.}
	\end{equation}
	Additionally, there exists a sequence of step sizes~$\eta_k > 0$ such that~$W_2^2(\mu_k,\mu^\star) \leq R_0^2$ and~$\E[\norm{\lambda_k}^2] + \E[\norm{\nu_k}^2] < \infty$ for all~$k$. The same results hold~(without expectations) when using exact dual gradients, i.e., if the updates in steps~\ref{E:pdlmc_lambda}--\ref{E:pdlmc_nu} are replaced by~$\lambda_{k+1} = \big[ \lambda_{k} + \eta_k \E_{\mu_k}[g] \big]_+$ and~$\nu_{k+1} = \nu_{k} + \eta_k \E_{\mu_k}[h]$.

\end{theorem}

\Cref{T:main_convex}, whose proof is deferred to~\Cref{X:proof_convex}, implies rates similar to those for GDA schemes in finite-dimensional Euclidean optimization~(see, e.g.,~\cite{Nedic09a}). To recover those rates, however, we must bound the magnitudes of~$\lambda_k,\nu_k$. In~\cite{Nedic09a}, this is done by bounding the iterates in the algorithm itself, i.e., by projecting them onto the set~$\calD_r = \{(\lambda,\nu) \in \setR^I_+ \times \setR^J \mid \max(\norm{\lambda}^2,\norm{\nu}^2) \leq r\}$ and choosing~$r$ such that~$\Phi^\star \subseteq \calD_r$~(\Cref{T:strong_duality}(ii) ensures this is possible). We then incur a bias on the order of~$\eta$ in~\eqref{E:main_convex_fix} that vanishes in the decreasing step size setting of~\eqref{E:main_convex_dec}. Though convenient, this is not \emph{necessary} since there exists a sequence of step sizes such that both~$\E[\norm{\lambda_k}^2]$ and~$\E[\norm{\nu_k}^2]$ are bounded for all~$k \geq 0$. In the interest of generality, Theorem~\ref{T:main_convex} holds without these hypotheses. It is worth noting that though faster rates and last iterates guarantees can be obtained for Euclidean saddle-point problems, they rely on more complex schemes than the GDA in~\Cref{L:pdlmc} involving acceleration or proximal methods~\cite{nemirovski2004prox, nesterov2007dual, lin2020near, mokhtari2020unified}.

The results in~\Cref{T:main_convex} are stated for the stochastic scheme in~\Cref{L:pdlmc}. However, \Cref{T:main_convex} yields the same rates~(without expectations) for exact dual gradients, i.e., for the dual ascent scheme~\eqref{E:simultaneous_dual_ascent}. In this case, the second condition in Assumption~\ref{A:bounded_gradients} simplifies to~$\max_k ( \| \E_{\mu_k}[g] \|^2, \| \E_{\mu_k}[h] \|^2 ) \leq G^2$. Not only are these milder assumptions than~\cite[Eq.~(16)]{Liu21s}, but the guarantees hold for discrete- rather than continuous-time dynamics. Finally, \eqref{E:main_convex_fix}--\eqref{E:main_convex_dec} imply convergence with respect to the KL divergence for convex potentials~($m = 0$) with stronger guarantees in Wasserstein metric for strongly convex ones~($m > 0$).

The convergence rates for distributions~$\mu_k$ from~\Cref{T:main_convex} also imply convergence rates for empirical averages across iterates~$x_k$ of~\Cref{L:pdlmc}. This corollary is obtained by combining~\eqref{E:main_convex_fix}--\eqref{E:main_convex_dec} with the following proposition. By taking~$\varphi$ to be the constraint functions~$g$ or~$h$ from~\eqref{P:constrained_sampling} yields feasibility guarantees for~\pdlmc~.

\begin{proposition}\label{T:ergodic_feasibility}
Consider samples~$x_k$ distributed according to~$\mu_k$ and~$c_k \geq 0$ with~$\sum_{k=1}^K c_k = 1$ such that
$
	\sum_{k = 1}^K c_k \Big[ \KL(\mu_{k} \,\|\, \mu^\star)
		+ \frac{m}{2} W_2^2(\mu_k,\mu^\star)
	\Big] \leq \Delta_K
$ for~$m \geq 0$.
Then, it holds that
\begin{equation*}
	\abs{\E \!\Bigg[ \sum_{k=1}^K c_k \varphi(x_k) \Bigg]
			- \E_{\mu^\star}[\varphi]}
	\leq
	\begin{cases}
		\sqrt{2\Delta_K} \text{,} &\text{if $\varphi$ is bounded by $1$,}
		\\
		\sqrt{\dfrac{2 \Delta_K}{m}} \text{,} &\text{if $\varphi$ is $1$-Lipschitz and $m > 0$.}
	\end{cases}
\end{equation*}

\proof
See~\Cref{X:proof_convex}.
\end{proposition}

\subsection{PD-LMC with LSI potentials}
\label{S:lsi_convergence}

In this section, we replace Assumption~\ref{A:convexity} on the convexity of the potential by an LSI common in the sampling literature. We consider only inequality constraints~($J = 0$) here and omit the function arguments~$\nu$, since accounting for equality constraints requires significant additional assumptions.

\begin{assumption}\label{A:lsi}
	The distribution~$\mu_{\lambda}$ satisfies the LSI for bounded~$\lambda$, i.e.,  there exists~$\sigma > 0$ such~that $2 \sigma \KL(\zeta\|\mu_{\lambda}) \leq	    	\left\Vert \nabla \log \left( d\zeta/d\mu_{\lambda}\right) \right \Vert^2_{L^2(\zeta)}$ 	    	for all $\zeta \in \spprobmeas$.
\end{assumption}

The LSI in Assumption~\ref{A:lsi} is often used in the analysis of the standard LMC algorithm~\cite{Vempala19r, Chewi22a}. It holds, e.g., when~$f$ is strongly convex and~$g$ is a~(possibly non-convex) bounded function due to the Holley-Stroock perturbation theorem~\cite{holley1986logarithmic}. In fact, if~$f$ is $1$-strongly convex and~$\abs{g}$ is bounded by~$1$, then Assumption~\ref{A:lsi} holds for~$\sigma \geq e^{-2\lambda}$~(see, e.g.,~\cite[Prop. 5.1.6]{Bakry14a} or~\cite[Thm~1.1]{Cattiaux22f}). The LSI is akin to the Polyak-{\L}ojasiewicz~(PL) condition from Euclidean optimization~\cite{Karimi16l}, which supposes issues with GDA methods such as~\Cref{L:pdlmc}. Indeed, it is not enough for the Lagrangian~\eqref{E:lagrangian} to satisfy the PL~condition in the primal variable to guarantee the convergence of GDA in Euclidean spaces. We must either modify~\Cref{L:pdlmc} using acceleration or proximal methods~\cite{lin2020near, yang2022nest, boroun2023accelerated, sanjabi2018convergence} or impose the PL~condition also on~$\lambda$~\cite{yang2020global, fiez2021global}. Since the Lagrangian~\ref{E:lagrangian} is linear in~$\lambda$, it is clear that Algorithm~\ref{L:pdlmc} will not suffice to provide theoretical guarantees in the LSI case.

We therefore consider the variant in~\Cref{L:dlmc}, where~$N^0_k$ LMC iterations~(step~\ref{E:pdlmc_x}) are executed before updating the dual variables~(step~\ref{E:pdlmc_lambda}). This is akin to using different time-scales in continuous-time, a common technique for solving saddle-point problems~\cite{fiez2021global, yang2022nest}. Since it resembles a dual ascent counterpart of the LMC algorithm~\eqref{E:lmc}, we refer to it as \emph{(stochastic)~dual~LMC}~(DLMC). As opposed to the dual ascent algorithm from~\cite{Liu21s} in~\eqref{E:simultaneous_dual_ascent}, however, Algorithm~\ref{L:dlmc} does not require any explicit evaluation of expected values. The following theorem provides an analysis of its convergence.

\begin{algorithm}[t]
\caption{(Stochastic) dual LMC}
	\label{L:dlmc}
\begin{algorithmic}[1]
	\State \textbf{Inputs:} $N^0_k > 0$~(burn-in), $\gamma_k,\eta_k > 0$~(step sizes), and~$\lambda_{0} = 0$.

	\For{$k = 0,\dots,K-1$}
		\State $x_0 \sim \mu_0$
			\State $\displaystyle x_{n+1} = x_n - \gamma_k \nabla_{x} U(x_n, \lambda_k) + \sqrt{2\gamma_k}\,\beta_n$, \ \ $\beta_n \sim \calN(0, \Id)$,
			\quad for $n = 0,\dots,N^0_k-1$
			\label{E:dlmc_x}

		\smallskip
		\State $\displaystyle \lambda_{k+1} = \big[ \lambda_{k} + \eta_k g(x_{N^0_k}) \big]_+$
			\label{E:dlmc_lambda}
	\EndFor
\end{algorithmic}
\end{algorithm}

\begin{theorem}\label{T:main_lsi}
	Assume that the functions~$f,g$ are~$M$-smooth, i.e., have $M$-Lipschitz continuous gradients, satisfy Assumption~\ref{A:lsi}, and that~$\E_{\mu}[\norm{g}^2] \leq G^2$ for all~$\mu \in \spprobmeas$. Let~$0 < \eta_k \leq \eta$, $0 < \epsilon \leq \eta G^2 < 1$,
	\begin{equation*}
		\gamma_k = \gamma \leq \dfrac{\sigma \epsilon}{16 d M^2}
		\text{,} \quad \text{and} \quad
	    N^0_k \geq \dfrac{1}{\gamma \sigma} \log \left(
	    	\dfrac{2 \KL(\mu_0 \| \mu_{\lambda_k})}{\epsilon}
	    \right)
	    	\text{.}
	\end{equation*}
	Under Assumption~\ref{A:strict_feasibility}, there exists~$B < \infty$ such that the distributions~$\{\mu_k\}$ of the samples~$\{x_{N^0_k}\}$ generated by~\Cref{L:dlmc} satisfy
	\begin{equation}\label{E:main_lsi}
	    \frac{1}{K} \sum_{k = 0}^{K-1} \KL(\mu_k \| \mu^\star)
	    	\leq \epsilon + \dfrac{\eta G^2}{2} + \dfrac{2 I B^2}{\eta K}
	    	\text{.}
	\end{equation}
	Recall from~\eqref{P:constrained_sampling} that~$I$ is the number of inequality constraints. Additionally, $\E[\norm{\lambda_k}_1]$ is bounded for all~$k$.
\end{theorem}

\Cref{T:main_lsi}, whose proof is deferred to~\Cref{X:proof_lsi}, provides similar guarantees as (approximate)~subgradient methods in finite-dimensional optimization~(see, e.g., \cite{sanjabi2018convergence, Chamon23c}). This is not surprising seen as~$\gamma_k,N^0_k$ in~\Cref{T:main_lsi} are chosen to ensure that step~\ref{E:dlmc_x} yields a sample~$x_{k} \sim \bar{\mu}_{k}$ such that~$\KL(\bar{\mu}_{k} \| \mu_{\lambda_k}) \leq \epsilon$ using~\cite[Theorem~1]{Vempala19r}. At this point, $g(x_{N^0_k})$ in step~\ref{E:dlmc_lambda} is an approximate, stochastic subgradient of the dual function~\eqref{E:dual_function}. Though it may appear from~\eqref{E:main_convex_fix} and~\eqref{E:main_lsi} that Algorithms~\ref{L:pdlmc} and~\ref{L:dlmc} have the same convergence rates, an informal computation shows that the latter evaluates on the order of~$d\kappa^2/\eta$ as many gradient per iteration, where~$\kappa = M/\sigma$. Note that we can once again apply~\Cref{T:main_lsi} to derive ergodic average and feasibility guarantees for Algorithm~\ref{L:dlmc}.

\section{Experiments}\label{S:experiments}

\begin{figure}
	\includegraphics[width=\linewidth]{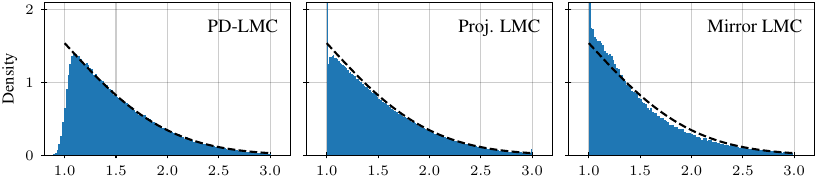}

	\caption{Sampling from a 1D truncated Gaussian~(ground truth displayed as dashed lines).}
	\label{F:1d_truncated_gaussian}
\end{figure}

\begin{figure}
	\centering
	\includegraphics{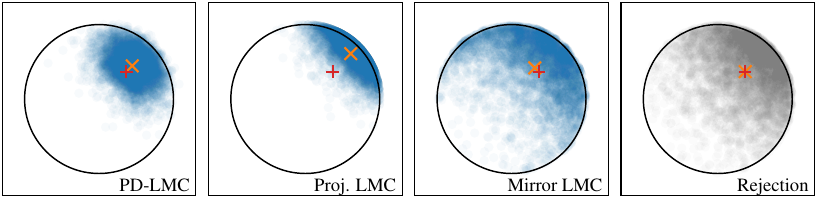}

	\caption{Sampling from a 2D truncated Gaussian (true mean in red and sample mean in orange).}
	\label{F:2d_truncated_gaussian}
\end{figure}

We now return to the applications described in~\Cref{S:csmp} to showcase the behavior of \pdlmc. We defer implementation details and additional results to~\Cref{X:applications}. Code for these examples is publicly available at~\url{https://www.github.com/lfochamon/pdlmc}.

\textbf{1. Sampling from convex sets.} We cast the problem of sampling from a Gaussian distribution~$\calN(0,1)$ truncated to~$\calC = [1,3]$ as~\eqref{P:constrained_sampling} by taking~$f(x) = x^2/2$ and~$g(x) = [(x-1)(x-3)]_+$~(see~\Cref{S:csmp}). Fig.~\ref{F:1d_truncated_gaussian} shows histograms for the samples obtained using~\pdlmc, the projected LMC~(Proj.\ LMC) from~\cite{Bubeck18s}, and the mirror LMC from~\cite{Ahn20e}, all with the same step size. Both Proj.\ LMC and Mirror LMC generate an excess of samples close to the boundary~(between~$1.5$ and~$3$ times more samples than expected). This leads to an underestimation of the mean~(Proj.\ LMC:~$1.488$ / Mirror LMC:~$1.470$ vs.\ true mean:~$1.510$). In contrast, \pdlmc provides a more accurate estimate~($1.508$). Yet, since it constrains the distribution~$\mu$ rather than its samples, it is not an \emph{interior-point method} and can produce samples outside of~$\calC$. Theorems~\ref{T:main_convex}--\ref{T:main_lsi} show that this becomes less frequent as the algorithm progresses~(in Fig.~\ref{F:1d_truncated_gaussian}, only~$2\%$ of the samples are not in~$\calC$). This occurs even without using \emph{mini-batches} in steps~\ref{E:pdlmc_lambda}--\ref{E:pdlmc_nu} of~\Cref{L:pdlmc} as in~\cite{Liu21s}. In fact, our experiments show that \emph{mini-batches} increase the computational complexity with no performance benefit~(\Cref{X:applications}). These issues are exacerbated in more challenging problems, such as sampling from a two-dimensional standard Gaussian centered at~$[2,2]$ restricted to an unit~$\ell_2$-norm ball~(Fig.~\ref{F:2d_truncated_gaussian}). In this case, Proj.\ LMC places almost~$25\%$ of its samples on the boundary~(where only~$0.14\%$ of samples should be), while \pdlmc only places~$1.8\%$ of its samples outside of the support. Mirror LMC provides a better mean estimation in this setting, although a bit more asymmetric than \pdlmc~[Mirror LMC: $(0.312,0.418)$ vs.\ \pdlmc: $(0.446,0.444)$ vs.\ true mean: $(0.368,0.368)$].

\textbf{2. Rate-constrained Bayesian models.} Here, we consider~$\pi$ to be the posterior of a Bayesian logistic regression model for the Adult dataset from~\cite{Dua17u}, where the goal is to predict whether an individual makes more than~$\$50$k based on socioeconomic information~(details on data pre-processing can be found in~\cite{Chamon20p}). We consider a standard Gaussian prior on the parameters~$\theta \in \setR^{d+1}$ of the model, where~$d$ is the number of features. Using the LMC algorithm to sample from the posterior~(i.e., no constraints), we find that while the average probability of positive predictions is~$19.1\%$ over the whole test set, it is~$26.2\%$ among males and~$5\%$ among females~(``Unconstrained'' in Fig.~\ref{F:fairness_main}).
To overcome this disparity, we take \emph{gender} to be the protected class in~\eqref{P:fairness}, constraining both~$\calG_\text{male}$ and~$\calG_\text{female}$ with~$\delta = 0.01$. Using \pdlmc, we obtain a Bayesian model that leads to an average probability of positive outcomes of~$18.1\%$ and~$15.1\%$ for males and females respectively. In fact, we now observe a substantial overlap of the distributions of positive predictions across genders for the constrained posterior~$\mu^\star$~(``Constrained ($\delta=0.01$)'' in Fig.~\ref{F:fairness_main}). This substantial reduction of prediction disparities comes at only a minor decline in accuracy~(unconstrained: $84\%$ vs constrained: $82\%$).

\begin{figure}[t]
	\begin{minipage}[b]{0.49\columnwidth}
		\includegraphics[width=\linewidth]{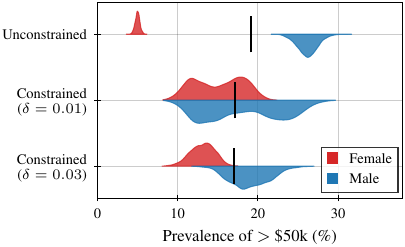}
		\caption{Distribution of the probability of~predicting $>\,$50k under the Bayesian logistic model (black lines indicate the mean across genders).}
		\label{F:fairness_main}
	\end{minipage}
	\hfill
	\begin{minipage}[b]{0.49\columnwidth}
		\centering
		\includegraphics[width=\linewidth]{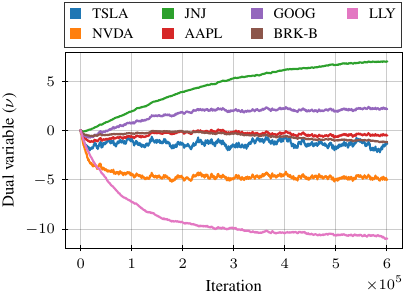}

		\caption{Counterfactual sampling of the stock market: dual variables.\\}
		\label{F:stock_return}
	\end{minipage}
\end{figure}

\textbf{3. Counterfactual sampling.} Though the distribution of positive predictions changes considerably for both male and female individuals, the final dual variables~($\lambda_\text{male} = 0$ and~$\lambda_\text{female} \approx 160$) show that these changes are due uniquely to the \emph{female} group~[as per Prop.~\ref{T:strong_duality}(iv)]. This implies that the reference model~$\pi$ is itself compatible with the requirement for the male group, but that reducing the disparity for females requires considerable deviations from it. By examining~$\lambda_\text{female}$, we conclude \emph{without recalculating~$\mu^\star$} that even small changes in the tolerance~$\delta$ for the female constraint would substantially change the distribution of outcomes~[Prop.~\ref{T:strong_duality}(v)]. This is confirmed by ``Constrained ($\delta=0.03$)'' in Fig.~\ref{F:fairness_main}. Notice that this is only possible due to the primal-dual nature of \pdlmc. This type of counterfactual analysis is even more beneficial in the presence of multiple requirements. Indeed, let~$\pi$ be the posterior of a Bayesian model for the daily (log-)return of a set of assets~(see Appendix~\ref{X:applications} for more details). Using~\eqref{P:porfolio}, we consider how the market would look like if the average (log-)return of each asset were to have been~(exactly)~$20\%$ higher. Inspecting the dual variables~(Fig.~\ref{F:stock_return}), we notice that this increased market return is essentially driven by two stocks: \texttt{NVDA} and~\texttt{LLY}~($\nu < 0$). In fact, the reference model~$\pi$ would be consistent with an even higher increase for~\texttt{JNJ} and~\texttt{GOOG}~($\nu > 0$). We confirm these observations by constraining only~\texttt{NVDA} and~\texttt{LLY}, which yields essentially the same (log-)return distribution for all assets.

\section{Conclusion}

We tackled the problem of sampling from a target distribution while satisfying a set of statistical constraints. Based on a GDA method in Wasserstein space, we put forward a fully stochastic, discrete-time primal-dual LMC algorithm~(\pdlmc) that precludes any explicit integration in its updates. We analyze the behavior of \pdlmc for (strongly)~convex and log-Sobolev potentials, proving that the distribution of its samples converges to the optimal constrained distribution. We illustrated the use of \pdlmc for different constrained sampling applications. Future work include strengthening the convergence results to almost sure guarantees and improving the rates obtained using proximal and extra gradient methods, particularly in the LSI setting.

\begin{ack}
	The work of L.F.O. Chamon is funded by the Deutsche Forschungsgemeinschaft~(DFG, German Research Foundation) under Germany’s Excellence Strategy~(EXC 2075-390740016).
\end{ack}

\bibliographystyle{IEEEtran}
\bibliography{aux_files/biblio,aux_files/mrk,aux_files/ml,aux_files/control,aux_files/stat}

\newpage
\appendix

\section{Related work}
	\label{X:related_work}
\input{app_related}

\newpage
\section{Wasserstein space and discrete-time flows}
	\label{X:background_wass}
\input{app_wass}

\newpage
\section{Proofs from Section~\ref{S:convex_convergence}}
	\label{X:proof_convex}
\input{app_proof_cvx}

\newpage
\section{Proofs from Section~\ref{S:lsi_convergence}}
	\label{X:proof_lsi}
\input{app_proof_lsi}

\newpage
\section{Applications}
	\label{X:applications}
\input{app_applications}

\newpage
\section{Example application of Prop.~\ref{T:strong_duality}}
	\label{X:gaussian_duality}
\input{app_gaussian_duality}

\newpage
\section*{NeurIPS Paper Checklist}

\input{app_checklist}

\end{document}

%% file: packages.tex
\usepackage{amsfonts, amssymb, amsmath, amsthm, bm, accents, dsfont}
\usepackage{fontawesome5}

\usepackage{hyperref, url}

\usepackage{graphicx, tabularx, caption}

\usepackage{blindtext}

\usepackage[english]{babel}

\makeatletter
\protected\def\abx@missing#1{%
  \mbox{\reset@font\color{blue}#1}}
\makeatother

\usepackage{algpseudocode,algorithm}
\algrenewcommand\algorithmicdo{}
\algrenewtext{EndFor}{\algorithmicend}
\algrenewtext{EndProcedure}{\algorithmicend}

\usepackage{xcolor}

\definecolor{stutt.blue}{RGB}{0,81,158}
\definecolor{stutt.lightblue}{RGB}{0,190,255}
\definecolor{stutt.gray}{RGB}{62, 68, 76}
\definecolor{stutt.darkblue}{RGB}{0,50,98}

\definecolor{Set1.red}{rgb}{0.894117647058824,0.101960784313725,0.109803921568627}
\definecolor{Set1.blue}{rgb}{0.215686274509804,0.494117647058824,0.721568627450980}
\definecolor{Set1.green}{rgb}{0.301960784313725,0.686274509803922,0.290196078431373}
\definecolor{Set1.purple}{rgb}{0.596078431372549,0.305882352941177,0.639215686274510}
\definecolor{Set1.orange}{rgb}{1,0.498039215686275,0}
\definecolor{Set1.yellow}{rgb}{1,1,0.2}
\definecolor{Set1.brown}{rgb}{0.650980392156863,0.337254901960784,0.156862745098039}
\definecolor{Set1.pink}{rgb}{0.968627450980392,0.505882352941176,0.749019607843137}
\definecolor{Set1.grey}{rgb}{0.6,0.6,0.6}

\usepackage{microtype}

\usepackage{array}		
\usepackage{tabularx}
\usepackage{longtable}
\usepackage{booktabs}		
\usepackage[inline,shortlabels]{enumitem}		
\setenumerate{itemsep=\smallskipamount,topsep=\smallskipamount,left=\parindent}
\setitemize{itemsep=\smallskipamount,topsep=\smallskipamount,left=\parindent}
\usepackage{makecell}

\usepackage{mathtools}

\usepackage[sort&compress,capitalize,nameinlink]{cleveref}
\crefname{assumption}{Assumption}{Assumptions}

\theoremstyle{plain}
\newtheorem{theorem}{Theorem}[section]
\newtheorem{proposition}[theorem]{Proposition}
\newtheorem{lemma}[theorem]{Lemma}

\theoremstyle{definition}

\newtheorem{assumption}[theorem]{Assumption}
\theoremstyle{remark}

\usepackage{enumitem}

%% file: mysymbol.tex
\let\Pr\relax
\DeclareMathOperator{\Pr}{\mathbb{P}}
\DeclareMathOperator{\E}{\mathbb{E}}

\DeclareMathOperator{\KL}{KL}

\DeclareMathOperator{\diag}{diag}

\newcommand{\norm}[1]{\ensuremath{\left\| #1 \right\|}}
\newcommand{\abs}[1]{\ensuremath{{\left\vert #1 \right\vert}}}

\DeclareMathOperator{\indicator}{\mathbb{I}}

\DeclareMathOperator*{\argmin}{argmin}
\DeclareMathOperator*{\argmax}{argmax}
\DeclareMathOperator*{\minimize}{minimize}

\DeclareMathOperator{\subjectto}{subject\ to}

\newcommand{\calC}{\ensuremath{\mathcal{C}}}
\newcommand{\calD}{\ensuremath{\mathcal{D}}}
\newcommand{\calE}{\ensuremath{\mathcal{E}}}
\newcommand{\calF}{\ensuremath{\mathcal{F}}}
\newcommand{\calG}{\ensuremath{\mathcal{G}}}
\newcommand{\calH}{\ensuremath{\mathcal{H}}}

\newcommand{\calL}{\ensuremath{\mathcal{L}}}

\newcommand{\calN}{\ensuremath{\mathcal{N}}}

\newcommand{\calP}{\ensuremath{\mathcal{P}}}

\newcommand{\calV}{\ensuremath{\mathcal{V}}}

\newcommand{\calX}{\ensuremath{\mathcal{X}}}

\newcommand{\bbR}{\ensuremath{\mathbb{R}}}

\def\st/{\textsuperscript{st}}
\def\nd/{\textsuperscript{nd}}
\def\rd/{\textsuperscript{rd}}
\def\th/{\textsuperscript{th}}

\newcommand{\setR}{\bbR}

\newcommand{\del}{\ensuremath{\partial}}

\newcommand{\spprobmeas}{\mathcal{P}_2(\setR^d)}

\makeatletter
\def\nnil{\nil}
\newcounter{prob}
\newenvironment{prob}[1][\nil]{%
	\def\tmp{#1}
	\equation
	\ifx\tmp\nnil
		\refstepcounter{prob}
		\tag{P\Roman{prob}}
	\else
		\tag{\tmp}
	\fi
	\aligned%
}{%
	\endaligned\endequation%
}
\makeatother

\newenvironment{prob*}{%
	\csname equation*\endcsname%
	\aligned%
}{%
	\endaligned%
	\csname endequation*\endcsname%
}

\newcommand{\Id}{\mathrm{I}_d}
\newcommand{\I}{\mathrm{I}}

\newcommand{\mut}{\ensuremath{\tilde{\mu}}}

\newcommand{\pdlmc}{PD\nobreakdash-LMC\xspace}

%% file: app_related.tex

In constrained sampling, it is important to distinguish between two types of constraints: support constraints and statistical constraints. The former deals with sampling from a target distribution~$\pi$ that is supported on a proper subset~$\calX \subset \setR^d$, which arises in applications such as latent Dirichlet allocation~\cite{blei2003latent} and regularized regression \cite{celeux2012regularization}. The latter is the problem tackled in the current work.

A first family of constrained sampling methods relies on rejection sampling: it obtains samples via any~(unconstrained) method, rejecting those that violate the constraint~(see, e.g., \cite{Lang07b, Li15e}). Though this approach can handle constraints of any nature, it is often inefficient in terms of number of samples generated per iteration of the method~(effective number of samples), especially when confronted with intricate constraints and high dimensional problems.

These drawbacks can be addressed for support constraints using techniques inspired by finite-dimensional constrained optimization. Projected LMC, for instance, deals with the problem of sampling from a target distribution restricted to a convex set~\cite{Bubeck18s, Lamperski21p}. Barrier methods have also been used to tackle the same problem~\cite{Li22s, Noble23u}. Similarly, mirror and proximal descent versions of popular sampling algorithms such as LMC~\cite{Hsieh18m, Zhang20w, Salim20p, Ahn21e, Jiang21m, Srinivasan23f} and Stein Variational Gradient Descent~(SVGD)~\cite{Shi22s} have been proposed. Mirror descent algorithms enforce constraints by mapping~(mirroring) the primal variables to a space with a different geometry~(induced by a Bregman divergence) over which the optimization is carried out. Alternatively, methods adapted to target distributions support on manifolds have also been put forward~\cite{Girolami11r, Brubaker12a, Noble23u}. In practice, these methods require explicit expressions for the projections, barriers, and mirror maps describing the constraints and are therefore not adapted to handle statistical requirements such as those considered in~\eqref{P:constrained_sampling}. Langevin dynamics with constraint violation penalties were considered in~\cite{Gurbuzbalaban22p}, although they cannot enforce exact constraint satisfaction.

Statistical~(moment) constraints such as those considered in~\eqref{P:constrained_sampling} were investigated in~\cite{Liu21s}. As we discussed at the end of Section~\ref{S:lagrangian}, this paper considers the combination of LMC and approximate dual ascent shown in~\eqref{E:simultaneous_dual_ascent}. It also introduces a similar version of SVGD as well as algorithms based on barriers. Aside from requiring exact integration against intractable measures~(namely,~$\mu_k$), convergence guarantees for these methods hold under restrictive assumptions on the constraints~$g_i$. Additionally, guarantees are derived only for continuous-time~(gradient flows) dynamics.

This work is also closely related to saddle-point methods in finite-dimensional optimization. For the general problem of~$\max_x\ \min_y\ f(x,y)$, the behavior of descent-ascent methods have been investigated under a myriad of scenarios, including for functions~$f$ that are (strongly)~convex/(strongly)~concave~\cite{tseng1995linear, nemirovski2004prox, nesterov2007dual, Nedic09a, lin2020near, mokhtari2020unified, golowich2020last} as well as non-convex/non-concave under, e.g., PL conditions~\cite{lin2020near, yang2020global, fiez2021global, yang2022nest, boroun2023accelerated}. In general, convergence holds for the ergodic average of iterates~\cite{Nedic09a, yang2020global, golowich2020last}. Last-iterate results often require different algorithms, involving proximal point or extra gradient methods and time scale separation~\cite{lin2020near, golowich2020last, mokhtari2020unified}. In particular, guarantees for the GDA method used in~\Cref{L:pdlmc} often requires stringent conditions that are hard to enforce for dual problems such as~\eqref{P:dual}.

%% file: app_wass.tex

In this subsection we give some background on the Wasserstein spaces and sampling as optimization. For $\mu,\nu \in \spprobmeas$, we define the 2-Wasserstein distance as
\begin{equation*}
	W_2^2(\mu,\nu)= \inf_{s \in \mathcal{S}(\mu,\nu)} \int \|x-y\|^2 ds(x,y),
\end{equation*}
where $\mathcal{S}(\mu,\nu)$ is the set of couplings between $\mu$ and $\nu$. The metric space~$(\spprobmeas,W_2)$ is referred to as the \emph{Wasserstein space}~\cite{panaretos2020invitation}. It can be equipped with a Riemannian structure~\cite{Otto01t}. In this geometric interpretation, the tangent space to~$\spprobmeas$ at~$\mu$ is included in~$L^2(\mu)$ and is equipped with a scalar product defined for~$f,g \in L^2(\mu)$ as
\begin{equation*}
\langle f,g \rangle_{L^2(\mu)}=\int f(x)g(x)d\mu(x).
\end{equation*}
We use its differential structure as introduced in \cite{Otto01t} and \cite[Chapter 10]{Ambrosio05g}. For the functionals at stake in this paper~(e.g., potential energies and negative entropy), the set of subgradients of a functional~$\mathcal{F}:\mathcal{P}_2(\setR^d)\to \setR$ at~$\mu\in \mathcal{P}(\setR^d)$ is non-empty if~$\mu$ satisfies some Sobolev regularity~\cite[Theorem 10.4.13]{Ambrosio05g}.

A \emph{Wasserstein gradient flow} of $\mathcal{F}$ is a solution $(\mu(t))_{t \in (0, T)}$, $T > 0$, of the continuity equation
\begin{equation*}
\frac{\partial \mu(t)}{\partial t} + \nabla \cdot (\mu(t) v(t)) = 0
\end{equation*}
that holds in the distributional sense, where $v(t)$ is a subgradient of $\mathcal{F}$ at $\mu(t)$. Among the possible processes $ (v(t))_t $, one has a minimal $ L^2(\mu(t)) $ norm and is called the velocity field of $ (\mu(t))_t $.
In a Riemannian interpretation of the Wasserstein space \cite{Otto01t},
this minimality condition can be characterized by $v(t)$ belonging to the tangent space to $\spprobmeas$ at $\mu(t)$ denoted $ T_{\mu(t)} \spprobmeas $, which is a subset of $ L^2(\mu(t))$. The Wasserstein gradient is defined as this unique element, and is denoted  $\nabla_{W_2}\mathcal{F}(\mu(t))$. In particular, if $\mu\in \spprobmeas$ is absolutely continuous with respect to the Lebesgue measure, with density in $ C^1(\setR^d)$ and such that $\mathcal{F}(\mu)<\infty$, $
\nabla_{W_2}\mathcal{F}(\mu)(x)=\nabla\mathcal{F}'(\mu)(x)$ for $\mu$-a.e.  $x \in \setR^d$, where $\mathcal{F}'(\mu)$ denotes
the first variation of $\mathcal{F}$ evaluated at $\mu$, i.e. (if it exists) the unique function $\mathcal{F}'(\mu):\setR^d \rightarrow \setR$ s.t.
\begin{equation}\label{eq:first_var}
\lim_{\epsilon \rightarrow 0}\frac{1}{\epsilon}(\mathcal{F}(\mu+\epsilon  \xi) -\mathcal{F}(\mu))=\int
\mathcal{F}'(\mu)(x)d \xi(x)
\end{equation}
for all $\xi=\nu-\mu$, where $\nu \in \spprobmeas.$

Now, we denote by~$ T_{\#}\mu$ the pushforward measure of~$\mu\in \spprobmeas$ by the measurable map~$T$. We recall that the KL divergence of $\mu$ relative to $\pi$ can be decomposed as~\eqref{E:kl}.
The distribution~$\mu_k$ of~$x_k$ in~\eqref{E:lmc} is known to follow a ``forward-flow'' splitting scheme~\citep{Wibisono18s} of the Fokker-Planck equation in~\eqref{E:langevin_diffusion}, namely
\begin{equation}\label{E:splitting}\begin{aligned}	\mu_{k+1/2} &= \big[\I - \gamma_k \nabla_{W_2} \calV(\mu_k) \big]_{\#}\mu_k	&&\text{(forward discretization for $\calV$)}	\\	\mu_{k+1} &= \mu_{k+1/2} \star \calN(0, 2 \gamma_k \Id)	&&\text{(exact flow for $\calH$)}		\text{,}
\end{aligned}\end{equation}
where~$\I$ denotes the identity map in~$L^2(\mu_k)$ and~$\nabla_{W_2} \calV(\mu_k) = \nabla_x f(x)$.

%% file: app_proof_cvx.tex

\begin{proof}[Proof of Theorem~\ref{T:main_convex}]

Consider the potential
\begin{equation}\label{E:lyapunov}
    V(\mu,\lambda,\nu) \triangleq W_2^2(\mu, \mu^\star)
    	+ \min_{(\lambda^\star,\nu^\star) \in \Phi^\star} \| \lambda - \lambda^\star \|^2
    	+ \| \nu - \nu^\star \|^2
    	\text{,}
\end{equation}
where~$\mu^\star$ is the solution of the constrained sampling problem~\eqref{P:constrained_sampling} and~$\Phi^\star$ is the set of solutions of the dual problem~\eqref{P:dual}. Our goal is to show that~$V$ decreases~(in some sense) along trajectories of Algorithm~\ref{L:pdlmc}. We say ``in some sense'' because contrary to the standard LMC algorithm, the distribution~$\mu_k$ of~$x_k$ is now a random measure that depends on the random variables~$\{\lambda_k,\nu_k\}$. Explicitly, we consider the filtration~$\calF_k = \sigma(\mu_0, \{\lambda_\ell,\nu_\ell\}_{\ell \leq k})$ and show that~$V$ decreases on average when conditioned on~$\calF_k$. This turns out to be enough to prove Theorem~\ref{T:main_convex}.

Indeed, notice that~$\mu_{k} \in \calF_k$. Hence, the potential energy $\calE(\mu_k,\lambda_k,\nu_k) \in \calF_{k}$ and the conditional law~$\mut_k = \calL(x_k \vert \calF_{k-1})$ evolves as in the regular LMC algorithm~\eqref{E:lmc}. That is to say, conditioned on the event~$\calF_k$, step~5 of Algorithm~\ref{L:pdlmc} follows a splitting scheme as ~\eqref{E:splitting}, i.e.,
\begin{subequations}\label{E:pdlmc}
\begin{align}
	\mut_{k+1/2} &= \big[\I - \eta_k \nabla_{W_2} \calE(\mu_k,\lambda_k,\nu_k) \big]_{\#}\mu_k
		\label{E:pdlmc_mu_grad}
	\\
	\mut_{k+1} &= \mut_{k+1/2} \star \calN(0, 2 \gamma \Id)
		\label{E:pdlmc_mu_flow}
\end{align}
\end{subequations}
Notice that all distributions in~\eqref{E:pdlmc} are now deterministic. The core of the proof is collected in the following lemma that shows that~$V$ is a non-negative supermartingale. Note that Lemma~\ref{T:lyapunov} describes the gap between ``half-iterations.'' This will be inconsequential for our purposes.

\begin{lemma}\label{T:lyapunov}
	Under the conditions of Theorem~\ref{T:main_convex}, we have
	\begin{multline}\label{E:supermartingale}
		\E\! \big[ V(\mut_{k+1/2},\lambda_{k+1},\nu_{k+1}) \vert \calF_k \big]
			\leq V(\mut_{k-1/2},\lambda_{k},\nu_{k})
		- 2\eta_k \bigg[
			\KL(\mut_{k} \,\|\, \mu^\star) + \dfrac{m}{2} W_2^2(\mut_k,\mu^\star)
		\bigg]
		\\
		{}+ \eta_k^2 \Big[
			\| \nabla_{W_2} \mathcal{E}(\mut_k,\lambda_k,\nu_k) \|^2_{L^2(\mut_k)}
			+ \E_{x \sim \mut_k} \norm{g(x)}^2 + \E_{x \sim \mut_k} \norm{h(x)}^2
		\Big]
		\text{.}
	\end{multline}
\end{lemma}

We defer the proof of Lemma~\ref{T:lyapunov} to show how it implies the bounds in Theorem~\ref{T:main_convex}. To do so, take the expectation of~\eqref{E:supermartingale} with respect to~$\{\lambda_\ell,\nu_\ell\}_{\ell \leq k}$ to obtain
\begin{multline}\label{E:avg_bound0}
	\E[ \Delta_k ] \triangleq
		\E\! \big[ V(\mut_{k+1/2},\lambda_{k+1},\nu_{k+1}) - V(\mut_{k-1/2},\lambda_{k},\nu_{k}) \big] \leq
	\\
	{}- 2\eta_k \E\! \bigg[
		\KL(\mut_{k} \,\|\, \mu^\star) + \dfrac{m}{2} W_2^2(\mut_k,\mu^\star)
	\bigg]
	\\
	{}+ \eta_k^2 \Big[
		\| \nabla_{W_2} \mathcal{E}(\mut_k,\lambda_k,\nu_k) \|^2_{L^2(\mu_k)}
		+ \E_{x \sim \mu_k} \norm{g(x)}^2 + \E_{x \sim \mu_k} \norm{h(x)}^2
	\Big]
	\text{.}
\end{multline}
Using the bounds in Assumption~\ref{A:bounded_gradients} then yields
\begin{multline}\label{E:avg_bound1}
	\E[ \Delta_k ] \triangleq
		\E\! \big[ V(\mut_{k+1/2},\lambda_{k+1},\nu_{k+1}) - V(\mut_{k-1/2},\lambda_{k},\nu_{k}) \big] \leq
	\\
	{}- 2\eta_k \E\! \bigg[
		\KL(\mut_{k} \,\|\, \mu^\star) + \dfrac{m}{2} W_2^2(\mut_k,\mu^\star)
	\bigg]
	+ \eta_k^2 \big( 3 + \E [\norm{\lambda_k}^2] + \E[\norm{\nu_k}^2] \big) G^2
	\text{.}
\end{multline}
Then, summing the LHS of~\eqref{E:avg_bound1} over~$k$ and using the fact that~$V$ is non-negative yields
\begin{equation}\label{E:delta_sum}
    \sum_{k = 1}^K \E[ \Delta_k ] =
    	\E\! \big[ V(\mut_{K+1/2},\lambda_{k+1},\nu_{k+1})  \big] - V(\mut_{1/2}, \lambda_{1}, \nu_{1})
    		\geq - \E [V(\mut_{1/2}, \lambda_{1}, \nu_{1})]
  		\text{.}
\end{equation}
Notice that the expectation here is taken only over~$\mu_0$ given that~$(\lambda_0,\nu_0) = (0,0)$ are deterministic. To proceed, we use the following proposition, whose proof we defer.
\begin{lemma}\label{T:R_0}
	Under the hypothesis of Theorem~\ref{T:main_convex} it holds that~$\E [V(\mut_{1/2}, \lambda_{1}, \nu_{1})] \leq R_0^2$ for
	\begin{equation*}
	    R_0^2 = W_2^2(\mu_{0},\mu^\star)
	    + \eta_0^2 \Big[
	    	\| \nabla f \|_{L^2(\mu_0)}^2 + \E_{\mu_0} [\| g \|^2] + \E_{\mu_0} [\| h \|^2]
	    \Big]
	    + \norm{\lambda^\star}^2 + \norm{\nu^\star}^2
	    	\text{.}
	\end{equation*}
\end{lemma}

Back in~\eqref{E:avg_bound1}, we obtain that
\begin{equation*}
	\sum_{k = 1}^K \eta_k \bigg[\!
		\E[\KL(\mut_{k} \,\|\, \mu^\star)]
		+ \dfrac{m}{2} \E [W_2^2(\mut_k,\mu^\star)]
	\!\bigg] \leq
	R_0^2
		+ \sum_{k = 1}^K \eta_k^2 3 G^2
	+ G^2 \sum_{k = 1}^K \eta_k^2 \big( \E[\norm{\lambda_k}^2] + \E[\norm{\nu_k}^2] \big)
		\text{.}
\end{equation*}
We conclude by using the convexity of the Wasserstein distance to write
\begin{equation*}
	\E \big[ W_2^2(\mut_{k+1/2},\mu^\star) \big] \geq W_2^2( \E[\mut_{k+1/2}],\mu^\star) \big]
		= W_2^2( \mu_{k+1/2},\mu^\star)
		\text{.}
\end{equation*}
Similarly for the KL divergence. The bounds in Theorem~\ref{T:main_convex} are then obtained for the specific choices of~$\eta_k$ by noticing that~$\sum_{k=1}^{K} 1/\sqrt{k} \geq \sqrt{K}$ and~$\sum_{k=1}^{K} 1/k \leq 1+\log(K)$. Notice that all inequalities in the proof continue to hold for deterministic~$\{\lambda_k,\nu_k\}$. The bounds in Theorem~\ref{T:main_convex} therefore also hold (without expectations) when using exact gradients to update the dual variables.

Finally, we show there exists a sequence of step sizes~$\eta_k$ such that~$\E[V(\mu_{k-1/2},\lambda_k,\nu_k)] \leq R_0^2$ for all~$k \geq 1$, where the expectation is taken over the $\{\lambda_k,\nu_k\}$. This immediately implies that~$W^2_2(\mu_k,\mu^\star) \leq R_0^2$ and both~$\E [\norm{\lambda_k}^2]$ and~$\E[\norm{\nu_k}^2]$ are bounded for all~$k$. We proceed by induction. The base case is covered by Lemma~\ref{T:R_0}. Suppose now that there exists a sequence~$\{\eta_0,\dots,\eta_{k-1}\}$ such that~$\E[V(\mut_{k-1/2},\lambda_k,\nu_k)] \leq R_0^2$. From the definition of~$V$ in~\eqref{E:lyapunov} and the fact that the~$(\lambda^\star,\nu^\star)$ are bounded~(Prop.~\ref{T:strong_duality}), we then obtain that~$\E[\norm{\lambda_k}],\E[\norm{\nu_k}]$ are bounded. Consequently, there exists~$\eta_{k} > 0$ such that
\begin{equation*}
    \eta_k^2 (3 + \E[\norm{\lambda_k}^2] + \E[\norm{\nu_k}^2]) G^2 \leq 2 \eta_k \E\!\Big[
   		\KL(\mut_{k} \,\|\, \mu^\star) + \dfrac{m}{2} W_2^2(\mut_k,\mu^\star)
   	\Big]
   		\text{.}
\end{equation*}
From~\eqref{E:avg_bound1}, we obtain that~$\E[\Delta_k] \leq 0$, which together with the induction assumption yields~$\E\! \big[ V(\mut_{k+1/2},\lambda_{k+1},\nu_{k+1}) \big] \leq R_0^2$.
\end{proof}

\begin{proof}[Proof of Lemma~\ref{T:lyapunov}]

The proof proceeds by combining two inequalities bounding the primal and dual terms in~\eqref{E:lyapunov}.

\medskip
\textbf{(i)~$\bm{W_2^2({\mut}_{k+1/2},\mu^\star)}$.}
We proceed following a coupling argument. Let~$s^k$ be an optimal coupling between the random variables~$Y \sim \mut_k$ and~$Z \sim \mu^\star$, i.e., a coupling that achieves~$W_2^2(\mut_k,\mu^\star)$. Consider now the random variable~$T = Y - \eta_k \nabla_{W_2} \mathcal{E}(\mut_k,\lambda_k,\nu_k)$ and observe from~\eqref{E:pdlmc_mu_grad} that it is distributed as~$\mut_{k+1/2}$. Naturally, the coupling~$s^k$ is no longer optimal for~$(T,Z)$, so that by the definition of the Wasserstein distance it follows that
\begin{equation}\label{E:W_bound_1}
 	W_2^2(\mut_{k+1/2},\mu^\star) \leq
	\int\!\Vert x - \eta_k \nabla_{W_2} \mathcal{E}(\mut_k, \lambda_k,\nu_k)
		- y \Vert^2 ds^k(x,y)
  \text{.}
\end{equation}
Expanding the RHS of~\eqref{E:W_bound_1} and using the $m$-strong convexity of~$\calE$~(Assumption~\ref{A:convexity}) yields
\begin{equation*}
\begin{aligned}
 	W_2^2(\mut_{k+1/2},\mu^\star) &\leq W_2^2(\mut_k,\mu^\star)
 		- \eta_k m W_2^2(\mut_k,\mu^\star)
 	\\
 	{}&+ 2\eta_k \big[ \mathcal{E}(\mu^\star,\lambda_k,\nu_k) - \mathcal{E}(\mut_k,\lambda_k,\nu_k) \big]
 	\\
 	{}&+ \eta_k^2 \| \nabla_{W_2} \mathcal{E}(\mut_k,\lambda_k,\nu_k) \|^2_{L^2(\mut_k)}
 		\text{.}
\end{aligned}
\end{equation*}
We can then bound the effect of the diffusion step using~\cite[Lemma~5]{Durmus19a} as in
\begin{equation}\label{E:flow_bound}
	W_2^2(\mut_{k},\mu^\star) - W_2^2(\mut_{k-1/2},\mu^\star) \leq
    	2\eta_k \big[ \mathcal{H}(\mu^\star) - \mathcal{H}(\mut_{k}) \big]
    	\text{,}
\end{equation}
which yields
\begin{equation}\label{E:grad_bound}
\begin{aligned}
 	W_2^2(\mut_{k+1/2},\mu^\star) &\leq W_2^2(\mut_{k-1/2},\mu^\star)
 		- \eta_k m W_2^2(\mut_k,\mu^\star)
 	\\
 	{}&+ 2\eta_k \big[ \mathcal{E}(\mu^\star,\lambda_k,\nu_k) + \mathcal{H}(\mu^\star)
 	- \mathcal{E}(\mut_k,\lambda_k,\nu_k) - \mathcal{H}(\mut_{k}) \big]
 	\\
 	{}&+ \eta_k^2 \| \nabla_{W_2} \mathcal{E}(\mut_k,\lambda_k,\nu_k) \|^2_{L^2(\mut_k)}
 		\text{.}
\end{aligned}
\end{equation}

\medskip
\textbf{(ii)~$\bm{\|\lambda_{k+1} - \lambda^\star\|^2 + \|\nu_{k+1} - \nu^\star\|^2}$.}
Notice that since~$\lambda^\star \in \setR_+^I$, the projection~$x \mapsto [x]_+$ is a contraction, i.e., $\Vert [\lambda]_{+} -\lambda^\star \Vert \leq \Vert \lambda - \lambda^\star \Vert$ for all~$\lambda \in \setR^I$. Using the definition of~$\lambda_{k+1},\nu_{k+1}$ from Algorithm~\ref{L:pdlmc}, we then obtain that
\begin{equation}\label{E:lambda_bound_1}
\begin{aligned}
    \|\lambda_{k+1} - \lambda^\star\|^2 + \|\nu_{k+1} - \nu^\star\|^2
    &\leq \|\lambda_{k} - \lambda^\star\|^2
    	+ 2\eta_k (\lambda_{k} - \lambda^\star)^\top g(x_k)
    	+ \eta_k^2 \norm{g(x_k)}^2
    \\
    {}&+ \|\nu_{k} - \nu^\star\|^2
       	+ 2\eta_k (\nu_{k} - \nu^\star)^\top h(x_k)
       	+ \eta_k^2 \norm{h(x_k)}^2
    	\text{,}
\end{aligned}
\end{equation}
for all~$(\lambda^\star,\nu^\star) \in \Phi^\star$. To proceed, consider the conditional expectation of~\eqref{E:lambda_bound_1} with respect to~$\calF_{k}$, namely,
\begin{multline}\label{E:lambda_bound_2}
	\E\! \big[ \|\lambda_{k+1} - \lambda^\star\|^2 \vert \calF_k \big]
		+ \E\! \big[ \|\nu_{k+1} - \nu^\star\|^2 \vert \calF_k \big]
		\leq \|\lambda_{k} - \lambda^\star\|^2 + \|\nu_{k} - \nu^\star\|^2
	\\
	{}+ 2\eta_k \Big[ (\lambda_{k} - \lambda^\star)^\top \E_{x \sim \mut_k} \big[ g(x) \big]
		+ (\nu_{k} - \nu^\star)^\top \E_{x \sim \mut_k} \big[ h(x) \big]
	\Big]
	\\
	{}+ \eta_k^2 \Big[ \E_{x \sim \mut_k} \norm{g(x)}^2
		+ \E_{x \sim \mut_k} \norm{h(x)}^2 \Big]
    	\text{,}
\end{multline}
where we used the fact that~$\mut_k,\lambda_k,\nu_k \in \calF_k$.
We conclude by using the linearity of the potential energy~$\calE$ from~\eqref{E:potential_entropy} in both~$\lambda$ and~$\nu$ to get
\begin{equation*}
    (\lambda_{k} - \lambda^\star)^\top \E_{x \sim \mut_k} \big[ g(x) \big]
    	+ (\nu_{k} - \nu^\star)^\top \E_{x \sim \mut_k} \big[ h(x) \big]
		= \calE(\mut_k,\lambda_k, \nu_k) - \calE(\mut_k,\lambda^\star,\nu^\star)
    	\text{.}
\end{equation*}
Back in~\eqref{E:lambda_bound_2}, we obtain
\begin{equation}\label{E:lambda_bound}
\begin{aligned}
	\E\! \big[ \|\lambda_{k+1} - \lambda^\star\|^2 \vert \calF_k \big]
		+ \E\! \big[ \|\nu_{k+1} - \nu^\star\|^2 \vert \calF_k \big]
		&\leq \|\lambda_{k} - \lambda^\star\|^2 + \|\nu_{k} - \nu^\star\|^2
	\\
    {}&+ 2\eta_k \big[
   		\calE(\mut_k,\lambda_k, \nu_k) - \calE(\mut_k,\lambda^\star,\nu^\star)
   	\big]
    \\
    {}&+ \eta_k^2 \Big[ \E_{x \sim \mut_k} \norm{g(x)}^2
   		+ \E_{x \sim \mut_k} \norm{h(x)}^2 \Big]
    	\text{,}
\end{aligned}
\end{equation}
for all~$(\lambda^\star,\nu^\star) \in \Phi^\star$.

\medskip
To proceed with the proof, combine~\eqref{E:grad_bound} and~\eqref{E:lambda_bound} to get
\begin{multline}\label{E:joint_bound}
 	W_2^2(\mut_{k+1/2},\mu^\star)
   		+ \E\! \big[ \|\lambda_{k+1} - \lambda^\star\|^2 \vert \calF_k \big]
   		+ \E\! \big[ \|\nu_{k+1} - \nu^\star\|^2 \vert \calF_k \big] \leq{}
	\\
 	W_2^2(\mut_{k-1/2},\mu^\star) + \|\lambda_{k} - \lambda^\star\|^2
 		 + \|\nu_{k} - \nu^\star\|^2
 	\\
 	{}+ 2\eta_k \bigg[ \mathcal{E}(\mu^\star,\lambda_k,\nu_k) + \mathcal{H}(\mu^\star)
 		- \calE(\mut_k,\lambda^\star, \nu^\star) - \mathcal{H}(\mut_{k})
 	\bigg] - \eta_k m W_2^2(\mut_k,\mu^\star)
 	\\
 	{}+ \eta_k^2 \Big[
 	 		\| \nabla_{W_2} \mathcal{E}(\mut_k,\lambda_k,\nu_k) \|^2_{L^2(\mut_k)}
 	 		+ \E_{x \sim \mut_k} \norm{g(x)}^2 + \E_{x \sim \mut_k} \norm{h(x)}^2
 	 	\Big]
 		\text{.}
\end{multline}
To upper bound the term in brackets, we add and subtract~$\log(Z)$ and use the decomposition of the Lagrangian in terms of~\eqref{E:potential_entropy} to obtain
\begin{equation*}
	\mathcal{E}(\mu^\star,\lambda_k,\nu_k) + \mathcal{H}(\mu^\star)
		- \mathcal{E}(\mut_k,\lambda^\star,\nu^\star) - \mathcal{H}(\mut_{k})
	= L(\mu^\star,\lambda_k,\nu_k) - L(\mut_k,\lambda^\star,\nu^\star)
		\text{.}
\end{equation*}
Using the saddle-point property~\eqref{E:saddle_point}, we then get
\begin{equation*}
	L(\mu^\star,\lambda_k,\nu_k) - L(\mut_k,\lambda^\star,\nu^\star) \leq
		L(\mu^\star,\lambda^\star,\nu^\star) - L(\mut_k,\lambda^\star,\nu^\star)
		\leq -\KL(\mut_{k} \,\|\, \mu^\star)
		\text{,}
\end{equation*}
We therefore conclude that
\begin{multline}\label{E:delta_bound1}
 	W_2^2(\mut_{k+1/2},\mu^\star)
   		+ \E\! \big[ \|\lambda_{k+1} - \lambda^\star\|^2 \vert \calF_k \big]
   		+ \E\! \big[ \|\nu_{k+1} - \nu^\star\|^2 \vert \calF_k \big] \leq{}
	\\
 	W_2^2(\mut_{k-1/2},\mu^\star) + \|\lambda_{k} - \lambda^\star\|^2
 		 + \|\nu_{k} - \nu^\star\|^2
 	- 2\eta_k \bigg[ \KL(\mut_{k} \,\|\, \mu^\star)
 	 		+ \dfrac{m}{2} W_2^2(\mut_k,\mu^\star)
 	 	\bigg]
 	 \\
  	{}+ \eta_k^2 \Big[
 		\| \nabla_{W_2} \mathcal{E}(\mut_k,\lambda_k,\nu_k) \|^2_{L^2(\mut_k)}
 		+ \E_{x \sim \mut_k} \norm{g(x)}^2 + \E_{x \sim \mut_k} \norm{h(x)}^2
 	\Big]
 		\text{.}
\end{multline}
Since~\eqref{E:delta_bound1} holds for all~$(\lambda^\star,\nu^\star) \in \Phi^\star$, it holds in particular for the minimizer of the RHS, for which we can then write~$V(\mut_{k-1/2},\lambda_{k},\nu_{k})$. By subsequently taking the minimum of the LHS, we obtain~\eqref{E:supermartingale}.
\end{proof}

\begin{proof}[Proof of Lemma~\ref{T:R_0}]
From the updates in Algorithm~\ref{L:pdlmc}, we obtain
\begin{equation*}
    \E[V(\mut_{1/2},\lambda_1,\nu_1)] \leq W_2^2(\mut_{1/2},\mu^\star)
    	+ \E [\| \big[ \eta_0 g(x_0) \big]_+ - \lambda^\star \|^2]
    	+ \E [\| \eta_0 h(x_0) - \nu^\star \|^2]
\end{equation*}
for all~$(\lambda^\star,\nu^\star) \in \Phi^\star$. Notice that since~$\lambda^\star \in \setR_+^I$, the projection~$x \mapsto [x]_+$ is a contraction, i.e., $\Vert [\lambda]_{+} -\lambda^\star \Vert \leq \Vert \lambda - \lambda^\star \Vert$ for all~$\lambda \in \setR^I$. Using the triangle inequality then yields
\begin{align*}
	\E[V(\mu_{1/2},\lambda_1,\nu_1)] &\leq W_2^2(\mut_{1/2},\mu_{0}) + W_2^2(\mu_{0},\mu^\star)
    	+ \E [\| \eta_0 g(x_0) - \lambda^\star \|^2]
    	+ \E [\| \eta_0 h(x_0) - \nu^\star \|^2]
    \\
    {}&\leq W_2^2(\mu_{0},\mu^\star) + \| \lambda^\star \|^2 + \| \nu^\star \|^2
    \\
    {}&+ W_2^2(\mut_{1/2},\mu_{0}) + \eta_0^2 \E [\| g(x_0) \|^2] + \eta_0^2 \E [\| h(x_0) \|^2]
    	\text{.}
\end{align*}

To proceed, observe from~\eqref{E:pdlmc_mu_grad} that~$\mut_{1/2} = \big[\Id - \eta_0 \nabla f \big]_{\#}\mu_0$, which implies that~$W_2^2(\mut_{1/2},\mu_{0}) \leq \eta_0^2 \| \nabla f \|_{L^2(\mu_0)}^2$. Using the bounds in Assumption~\ref{A:bounded_gradients}, we obtain that~$\E[V(\mut_{1/2},\lambda_1,\nu_1)] \leq R_0^2$.
\end{proof}

\begin{proof}[Proof of Proposition~\ref{T:ergodic_feasibility}]
Since~$c_k \geq 0$ and~$\sum c_k = 1$, we can use Jensen's inequality to write
\begin{equation*}
	\abs{\E \!\Bigg[ \sum_{k=1}^K c_k \varphi(x_k) \Bigg]
		- \E_{\mu^\star}[\varphi]}
		\leq \sum_{k=1}^K c_k \abs{\int \varphi d\mu_k - \int \varphi d\mu^\star}.
\end{equation*}
Using the relation between the~$\ell_1$- and~$\ell_2$-norm, we further obtain
\begin{equation}\label{E:l1l2_norm}
		\sum_{k=1}^K c_k \abs{\int \varphi d\mu_k - \int \varphi d\mu^\star}
		\leq \sqrt{\sum_{k=1}^K c_k
			\abs{\int \varphi d\mu_k - \int \varphi d\mu^\star}^2}
\end{equation}

If~$\varphi$ is bounded by~$1$, then the summands on the right-hand side of~\eqref{E:l1l2_norm} can be bounded by~$\text{TV}(\mu_k,\mu^\star)$. Indeed,
\begin{equation*}
	\abs{\int \varphi d\mu_k - \int \varphi d\mu^\star}
	\leq \int \varphi \abs{d\mu_k - d\mu^\star}
	\leq \int \abs{d\mu_k - d\mu^\star}
	= 2 \text{TV}(\mu_k,\mu^\star)
		\text{.}
\end{equation*}
The total variation distance can in turn be bounded by the KL divergence using Pinsker's inequality. We therefore obtain
\begin{equation*}
	\abs{\E \!\Bigg[ \sum_{k=1}^K c_k \varphi(x_k) \Bigg]
	- \E_{\mu^\star}[\varphi]}
		\leq \sqrt{2 \sum_{k=1}^K c_k \KL(\mu_{k} \| \mu^\star)}
		\leq \sqrt{2 \Delta_K}
		\text{.}
\end{equation*}

On the other hand, if~$\varphi$ is~$1$-Lipschitz, the summands on the right-hand side of~\eqref{E:l1l2_norm} are bounded by
\begin{equation*}
	\abs{\int \varphi d\mu_k - \int \varphi d\mu^\star}^2
		\leq W_1^2(\mu_k,\mu^\star) \le W_2^2(\mu_k, \mu^\star),
\end{equation*}
which implies that
\begin{equation*}
	\abs{\E \!\Bigg[ \sum_{k=1}^K c_k \varphi(x_k) \Bigg]
		- \E_{\mu^\star}[\varphi]}
		\leq \sqrt{\frac{2 \Delta_K}{m}}
		\text{,}
\end{equation*}
as long as~$m > 0$.
\end{proof}

%% file: app_proof_lsi.tex

\begin{proof}[Proof of Theorem~\ref{T:main_lsi}]

The proof is based on the analysis of the stochastic dual ascent algorithm
\begin{equation}\label{E:dual_ascent_sto}
	\lambda_{k+1} = \big[ \lambda_{k} + \eta_k g(\xi_k) \big]_+
		\text{,} \quad
		\text{for $\xi_k \sim \bar{\mu}_k$ such that
			$\KL(\bar{\mu}_k \| \mu_{\lambda_k}) \leq \epsilon$,}
\end{equation}
where~$\mu_{\lambda_k}$ is the Lagrangian minimizer from~\eqref{E:lagrangian} and~$\epsilon > 0$. Observe that, once again, we analyze the conditional distribution~$\bar{\mu}_k = \calL(\xi_k \vert \lambda_k)$. We collect this result in the following proposition:

\begin{proposition}\label{T:dual_ascent}
	Consider the iterations~\eqref{E:dual_ascent_sto} and assume that~$\E_{\mu}[g_i^2] \leq G^2$ for all~$\mu \in \spprobmeas$.
	Then, for~$0 < \eta_k \leq \eta$ and~$\epsilon \leq \eta G^2$, there exists~$B < \infty$ such that
	\begin{equation}\label{E:dual_ascent_sto_bound}
	    \frac{1}{K} \sum_{k = 0}^{K-1} \E [ \KL(\bar{\mu}_k \| \mu^\star) ]
	    	\leq \epsilon + \dfrac{\eta G^2}{2} + \dfrac{2 B^2 I}{\eta K}
	    	\text{.}
	\end{equation}
	The expectations are taken over the samples~$\xi_k \sim \bar{\mu}_k$.
\end{proposition}

We conclude by combining Prop.~\ref{T:dual_ascent} with~\cite[Theorem~1]{Vempala19r}, which characterizes the convergence of the LMC algorithm~\eqref{E:lmc} under Assumption~\ref{A:lsi}. Indeed, using the~$\gamma_k,N^0_k$ from Theorem~\ref{T:main_lsi} in Algorithm~\ref{L:pdlmc} guarantees that the law~$\bar{\mu}_k$ of~$x_k \vert \lambda_k$ is such that~$\KL(\bar{\mu}_k \| \mu_{\lambda_k}) \leq \epsilon$, i.e., satisfies the conditions in Prop.~\ref{T:dual_ascent}. We can then apply Jensen's inequality to get that~$\E [ \KL(\bar{\mu}_k \| \mu^\star) ] \geq \KL(\E[\bar{\mu}_k] \| \mu^\star) = \KL(\mu_k \| \mu^\star)$, where~$\mu_k$ is the law of~$x_{N^0_k}$ in Algorithm~\ref{L:dlmc}.
\end{proof}

\begin{proof}[Proof of Prop.~\ref{T:dual_ascent}]
The proof relies on the following lemmata:

\begin{lemma}\label{T:subgrad_approx}
For all~$\mu \in \spprobmeas$ such that~$\KL(\mu \| \mu_{\lambda}) \leq \epsilon$, the expected value~$E_{\mu}[g]$ is an \emph{approximate} subgradients of the dual function~$d$ in~\eqref{E:dual_function} at~$\lambda \in \setR^+_I$, i.e.,
\begin{equation}\label{E:subgrad_approx}
	d(\lambda) \geq d(\lambda^\prime)
		+ (\lambda - \lambda^\prime)^\top E_{\mu} [g]
		- \epsilon
		\text{,} \quad \text{for all $\lambda^\prime \in \setR^+_I$.}
\end{equation}
\end{lemma}

\begin{lemma}\label{T:norm_bound}
	Under the conditions of Prop.~\ref{T:dual_ascent}, it holds for all~$k$ that
	\begin{equation*}
		\|\lambda^\star\|_1 \leq B_0 \triangleq
			\frac{C - D^\star + \eta G^2 + \epsilon}{\delta}
		\quad \text{and} \quad
		\E [\|\lambda_{k} - \lambda^\star\|^2] \leq B^2 \triangleq 2 B_0^2 + 3 \eta^2 G^2
		\text{.}
	\end{equation*}
\end{lemma}

\begin{lemma}\label{T:complementary_slackness}
	The sequence~$(\lambda_k,g(\xi_k))$ obtained from~\eqref{E:dual_ascent_sto} is such that
	\begin{equation}\label{E:complementary_slackness}
		\frac{1}{K} \sum_{k = 0}^{K=1} \E \left[ \lambda_k^\top g(\xi_k) \right]
			\geq -\dfrac{\eta G^2}{2}
			\text{,}
	\end{equation}
	where the expectation is taken over realizations of~$\xi_k$.
\end{lemma}

Before proving these results, let us show how they imply Prop.~\ref{T:dual_ascent}. Start by noticing from~\eqref{E:dual_ascent_sto} that~$\lambda_{i,k+1} \geq \lambda_{i,k} + \eta g_i(\xi_k)$. Solving the recursion and recalling that~$\lambda_0 = 0$ then yields
\begin{equation*}
	\lambda_{i,K} \geq \eta \sum_{k = 0}^{K-1} g_i(\xi_k)
		\text{.}
\end{equation*}
Taking the expected value over~$\xi_k$ and dividing by~$\eta K$, we obtain
\begin{equation}\label{E:feasibility_1}
	\frac{1}{K} \sum_{k = 0}^{K-1} \E_{\bar{\mu}_k} [g_i]
		\leq \frac{\E [ \lambda_{i,K} ]}{\eta K}
		\leq \frac{\E [\abs{\lambda_{i,K} - \lambda_i^\star}] + \lambda_i^\star}{\eta K}
		\text{,}
\end{equation}
where the last bound stems from the triangle inequality. Since the upper bound is non-negative for all~$i$, we use the fact that the maximum of a set of values is less than the sum of those values to write
\begin{equation}\label{E:feasibility}
	\max_{i = 1,\dots,I} \Bigg[ \frac{1}{K} \sum_{k = 0}^{K-1} \E_{\bar{\mu}_k} [g_i] \Bigg]
		\leq \frac{\E [\norm{\lambda_{K} - \lambda^\star}_1] + \norm{\lambda^\star}_1}{\eta K}
		\leq \frac{\big(1+\sqrt{I}\big) B}{\eta K}
		\text{,}
\end{equation}
where we used Lemma~\ref{T:norm_bound} together with~$(\E[\norm{z}_1])^2 \leq I \cdot \E[\norm{z}^2]$ and~$B_0 < B$. Observe that~\eqref{E:feasibility} bounds the infeasibility of the ergodic average~$\frac{1}{K} \sum_{k = 0}^{K-1} \bar{\mu}_k$ for~$\bar{\mu}_k$ as in~\eqref{E:dual_ascent_sto}.

To proceed, use the relation between the normalization of~$\mu_\lambda$ and the dual function value~[see~\eqref{E:dual_function}] to decompose the KL divergence between~$\bar{\mu}_k$ and~$\mu^\star$ as
\begin{equation*}
	\KL(\bar{\mu}_k \| \mu^\star)
		= \KL(\bar{\mu}_k \| \mu_{\lambda_k})
		+ (\lambda^\star - \lambda_k)^\top \E_{\bar{\mu}_k}[g]
		+ d(\lambda_k) - d(\lambda^\star)
		\text{.}
\end{equation*}
Since~$\KL(\bar{\mu}_k \| \mu_{\lambda_k}) \leq \epsilon$ and~$d(\lambda) \leq d(\lambda^\star)$ for all~$\lambda \in \setR^I_+$, we get
\begin{equation*}
    \KL(\bar{\mu}_k \| \mu^\star) \leq \epsilon + (\lambda^\star - \lambda_k)^\top \E_{\bar{\mu}_k}[g]
\end{equation*}
Averaging over~$k$ and using H\"{o}lder's inequality then yields
\begin{equation*}
    \frac{1}{K} \sum_{k = 0}^{K-1} \KL(\bar{\mu}_k \| \mu^\star)
    	\leq \epsilon
    	+ \norm{\lambda^\star}_1 \cdot \max_{i = 1,\dots,I}\ %
    		\left[ \frac{1}{K} \sum_{k = 0}^{K-1} \E_{\bar{\mu}_k}[g_i] \right]
    	- \frac{1}{K} \sum_{k = 0}^{K-1} \lambda_k^\top \E_{\bar{\mu}_k}[g]
    	\text{,}
\end{equation*}
which from Lemma~\ref{T:norm_bound} and~\eqref{E:feasibility} becomes
\begin{equation*}
    \frac{1}{K} \sum_{k = 0}^{K-1} \KL(\bar{\mu}_k \| \mu^\star)
    	\leq \epsilon
    	+ \dfrac{\big(1+\sqrt{I} \big) B^2}{\eta K}
    	- \frac{1}{K} \sum_{k = 0}^{K-1} \lambda_k^\top \E_{\bar{\mu}_k}[g]
    	\text{,}
\end{equation*}
where we used~$B_0 < B$. Taking the expected value, applying Lemma~\ref{T:complementary_slackness}, and taking~$1+\sqrt{I} \leq 2 I$ for~$I \geq 1$ yields~\eqref{E:dual_ascent_sto_bound}.
\end{proof}

\subsection{Proof of Lemmata~\ref{T:subgrad_approx}--\ref{T:complementary_slackness}}

\begin{proof}[Proof of Lemma~\ref{T:subgrad_approx}]

From the definition of the Lagrangian~\eqref{E:lagrangian} and the dual function~\eqref{E:dual_function}~(with~$J = 0$), we obtain~$L(\mu, \lambda) = \KL(\mu\|\mu_{\lambda}) + d(\lambda)$. Using the fact that~$\KL(\mu \| \mu_{\lambda}) \leq \epsilon$, we get
\begin{equation}\label{E:subgrad_ineq1}
	0 \leq d(\lambda) - L(\mu, \lambda) + \epsilon
		\text{.}
\end{equation}
Again using the definition of the dual function~\eqref{E:dual_function}, we also obtain that~$d(\lambda^\prime) \leq L(\mu,\lambda^\prime)$. Adding to~\eqref{E:subgrad_ineq1} then gives
\begin{equation}\label{E:subgrad_ineq2}
	d(\lambda^\prime) \leq + d(\lambda)
		+ L(\mu,\lambda^\prime) - L(\mu, \lambda) + \epsilon
		\text{.}
\end{equation}
Notice from~\eqref{E:lagrangian} that the first term of the Lagrangians in~\eqref{E:subgrad_ineq2} cancel out, leading to~\eqref{E:subgrad_approx}.
\end{proof}

\begin{proof}[Proof of Lemma~\ref{T:norm_bound}]
Start by combining the update in~\eqref{E:dual_ascent_sto} and the fact that, since~$\lambda^\star \in \setR_+^I$, the projection~$x \mapsto [x]_+$ is a contraction, to get
\begin{equation*}
    \|\lambda_{k+1} - \lambda^\star\|^2 \leq \|\lambda_{k} - \lambda^\star\|^2
    	+ 2\eta (\lambda_{k} - \lambda^\star)^\top g(\xi_k)
    + \eta_k^2 \norm{g(\xi_k)}^2
    	\text{.}
\end{equation*}
Taking the conditional expectation given~$\lambda_k$ then yields
\begin{equation*}
    \E [\|\lambda_{k+1} - \lambda^\star\|^2 \mid \lambda_k]
    \leq \|\lambda_{k} - \lambda^\star\|^2
    	+ 2\eta (\lambda_{k} - \lambda^\star)^\top \E [g(\xi_k) \mid \lambda_k]
    + \eta^2 G^2
    	\text{.}
\end{equation*}
where we used the fact that~$\E_{\mu} [\norm{g}^2] \leq G^2$ for any~$\mu$. Noticing from~\eqref{E:dual_ascent_sto} that~$\E [g(\xi_k) \mid \lambda_k] = \E_{\bar{\mu}_k} [g]$, we can use Lemma~\ref{T:subgrad_approx} to get
\begin{equation}\label{E:lambda_bound0}
    \E [\|\lambda_{k+1} - \lambda^\star\|^2 \mid \lambda_k]
    \leq \|\lambda_{k} - \lambda^\star\|^2
    	+ 2\eta \Big[ d(\lambda_{k}) - D^\star + \eta G^2 + \epsilon \Big]
    	\text{,}
\end{equation}
where we used the fact that~$\eta/2 < \eta$ to simplify the relation.

To proceed, consider the set of approximate maximizers of the dual function
\begin{equation}\label{E:calD}
    \calD \triangleq \left\{
    	\lambda \in \setR^I_+ \mid d(\lambda) \geq D^\star - \eta G^2 - \epsilon
    \right\}
    	\text{.}
\end{equation}
Notice that~$\Phi^\star \subseteq \calD$. Since there exists at least one~$\lambda^\star$ that achieves~$D^\star$~(Prop.~\ref{T:strong_duality}), $\calD$ is not empty. Notice that for~$\lambda_k \notin \calD$, we have that~$d(\lambda_{k}) - D^\star + \eta G^2 + \epsilon < 0$. By the towering property, we therefore obtain from~\eqref{E:lambda_bound0} that
\begin{equation}\label{E:lambda_not_d}
    \E [\|\lambda_{k+1} - \lambda^\star\|^2 \mid \lambda_k \notin \calD] <
    	\E [\|\lambda_{k} - \lambda^\star\|^2 \mid \lambda_k \notin \calD]
    	\text{,}
\end{equation}
since~$\eta > 0$.

To bound the case when~$\lambda_k \in \calD$ we use the strictly feasible candidate~$\mu^\dagger$ from Assumption~\ref{A:strict_feasibility}. Indeed, recall that~$\KL(\mu^\dagger \| \pi) \leq C$ and~$E_{\mu^\dagger} [g_i] \leq -\delta < 0$ for all~$i$. From the definition of the dual function~\eqref{E:dual_function}, we obtain that
\begin{equation}
	d(\lambda) \leq L(\mu^\dagger,\lambda) \leq C - \norm{\lambda}_1 \delta
		\text{,}
\end{equation}
where we used the fact that~$\lambda \in \setR^I_+$ to write~$\sum_i \lambda_i = \norm{\lambda}_1$. Hence, it follows that
\begin{equation}\label{E:lambda_d_bound}
    \norm{\lambda}_1 \leq B_0 \triangleq \frac{C - D^\star + \eta G^2 + \epsilon}{\delta}
    	\text{,} \quad \text{for all } \lambda \in \calD
    	\text{.}
\end{equation}
Using the fact that~$\norm{z}^2 \leq \norm{z}_1^2$ for all~$z$ and that~$\Phi^\star \subset \calD$, we immediately obtain that~$\norm{\lambda - \lambda^\star}^2 \leq 2 B_0^2$ for all~$\lambda \in \calD$. Using the towering property and the fact that~$D^\star \geq d(\lambda_{k})$ and~$\epsilon \leq \eta G^2$ yields
\begin{equation}\label{E:lambda_d}
    \E [\|\lambda_{k+1} - \lambda^\star\|^2 \mid \lambda_k \in \calD] \leq
    	2 B_0^2 + 2 \eta \epsilon + \eta^2 G^2 \leq B^2
    	\text{.}
\end{equation}

To conclude, we write
\begin{align*}
    \E [\|\lambda_{k+1} - \lambda^\star\|^2] &=
    \E [\|\lambda_{k+1} - \lambda^\star\|^2 \mid \lambda_k \in \calD]
    	\Pr [\lambda_{k} \in \calD]
    \\
    {}&+ \E [\|\lambda_{k+1} - \lambda^\star\|^2 \mid \lambda_k \notin \calD]
    	\Pr [\lambda_{k} \notin \calD]
    	\text{.}
\end{align*}
Using~\eqref{E:lambda_not_d} and~\eqref{E:lambda_d} then yields
\begin{align*}
    \E [\|\lambda_{k+1} - \lambda^\star\|^2] &\leq B_1^2 \Pr [\lambda_{k} \in \calD]
    + \E [\|\lambda_{k} - \lambda^\star\|^2 \mid \lambda_{k} \notin \calD]
    	\Pr [\lambda_{k} \notin \calD]
    \\
    {}&\leq \max(B^2, \E [\|\lambda_{k} - \lambda^\star\|^2 \mid \lambda_{k} \notin \calD])
    	\text{.}
\end{align*}
Since both~\eqref{E:lambda_not_d} and~\eqref{E:lambda_d} holds independently of~$\lambda_{k+1}$, we can also write
\begin{equation*}
    \E [\|\lambda_{k} - \lambda^\star\|^2 \mid \lambda_k \notin \calD] \leq
    	\max(B^2, \E [\|\lambda_{k-1} - \lambda^\star\|^2 \mid \lambda_{k-1} \notin \calD])
\end{equation*}
Applying these relations recursively, we obtain that
\begin{equation*}
    \E [\|\lambda_{k+1} - \lambda^\star\|^2] \leq
    	\max(B^2, \|\lambda_{0} - \lambda^\star\|^2) = \max(B_1^2, \|\lambda^\star\|^2)
    	\text{.}
\end{equation*}
Noticing from~\eqref{E:lambda_d_bound} that since~$\lambda^\star \in \calD$ we have~$\norm{\lambda^\star}^2 \leq B_0^2 < B^2$ then concludes the proof.
\end{proof}

\begin{proof}[Proof of Lemma~\ref{T:complementary_slackness}]
To bound~\eqref{E:complementary_slackness}, we once again use the non-expansiveness of the projection to obtain
\begin{equation*}
	\Vert \lambda_{k+1} \Vert^2 \leq \Vert \lambda_k \Vert^2 + \eta^2 \Vert g(\xi_k) \Vert^2
			+ 2 \eta \lambda_k^\top g(\xi_k)
			\text{.}
\end{equation*}
Taking the expectation and using the fact that~$\E_{\mu}[\norm{g}^2] \leq G^2$, we get
\begin{equation*}
	\E [ \Vert \lambda_{k+1} \Vert^2 ] \leq
		\E [\Vert \lambda_k \Vert^2] + \eta^2 G^2
			+ 2 \eta \E[ \lambda_k^\top g(\xi_k) ]
			\text{.}
\end{equation*}
Applying this relation recursively from~$K$ and using the fact that~$\lambda_0 = 0$~(deterministic) yields
\begin{equation*}
	\E [\norm{\lambda_{K}}^2] \leq K\eta^2 G^2
		+ 2 \eta \sum_{k = 0}^{K-1} \E [\lambda_k^\top g(\xi_k)]
		\text{.}
\end{equation*}
Since~$\E [\norm{\lambda_{K}}^2] \geq 0$, we can divide by~$2 \eta K$ to obtain the desired result.
\end{proof}

%% file: app_applications.tex

In this section, we provide further details on the example applications described in~\Cref{S:csmp} as well as additional results from the experiments in~\Cref{S:experiments}. In these experimen$ $ts, we start all chains at zero~(unless stated otherwise) and use different step-sizes for each of the updates in steps~\ref{E:pdlmc_x}--\ref{E:pdlmc_nu} from \Cref{L:pdlmc}. We refer to them as~$\eta_x$, $\eta_\lambda$, and~$\eta_\nu$. In contrast, we do not use diminishing step-sizes.

\subsection{Sampling from convex sets}
\label{X:support}

\begin{table}[b]
\centering
\caption{Mean and variance estimates}
	\label{F:moments}
\begin{tabular}{ccccc}
\hline
 & True mean & Proj.~LMC & Mirror~LMC & PD-LMC \\
\hline
1D truncated Gaussian & 1.510 & 1.508 & 1.470 & 1.488 \\
2D truncated Gaussian  & $[0.368, 0.368]$ & $[0.611, 0.610]$ &
	$[0.312,0.418]$ & $[0.446, 0.444]$ \\
\hline
\end{tabular}
\end{table}

We are interested in sampling from target distribution
\begin{equation}\label{E:constrained_pi}
    \pi^o(x) \propto e^{-f(x)} \indicator(x \in \calC)
    	\text{,}
\end{equation}
for some closed, convex set~$\calC \subset \setR^d$. Several methods have been developed to tackle this problem, based on projections~\cite{Bubeck18s, Lamperski21p}, mirror maps~\cite{Hsieh18m, Ahn21e, Kook22s}, and barriers~\cite{Sharrock23l, Noble23u}. Here, we consider a constrained sampling approach based on~\eqref{P:constrained_sampling} instead.

To do so, note that sampling from~\eqref{E:constrained_pi} is equivalent to sampling from
\begin{prob*}
	\mu^\star \in \argmin_{\mu \in \calP_2(\calC)}& &&\KL(\mu \| \pi)
\end{prob*}
for the unconstrained~$\pi \propto e^{-f} \in \spprobmeas$. Note that this is exactly~\eqref{P:support_sampling}. Now let~$\calC$ be described by the intersection of the~$0$-sublevel sets of convex functions~$\{s_i\}_{i=1,\dots,I}$, i.e.,
\begin{equation*}
    \calC = \bigcap_{i = 1}^I \{x : s_i(x) \leq 0\}
    	\text{.}
\end{equation*}
Such a description is always possible by, e.g., considering the distance function $d(x,\mathcal{C}) = \inf_{y \in \mathcal{C}} \| x - y \|_2$, which is convex~(since~$\mathcal{C}$ is convex) and for which~$\{x : d(x,\mathcal{C}) \leq 0\} = \mathcal{C}$~(since~$\mathcal{C}$ is closed). Immediately, we see that solving~\eqref{P:support_sampling} is equivalent to solving
\begin{prob}\label{P:support_sampling_eq}
    \min_{\mu \in \spprobmeas}& &&\KL(\mu \| \pi)
    \\
    \text{subject to}& &&\E_{x\sim \mu} \!\big[ [ s_i(x) ]_+ \big] \leq 0
    	\text{,}\ \ i = 1,\dots,I
    \text{,}
\end{prob}
where~$[z]_{+} = \max(0,z)$. Note that~\eqref{P:support_sampling_eq} has the same form as~\eqref{P:constrained_sampling}.

To see why this is the case, consider without loss of generality that~$i = 1$. Since $[s_i(x)]_{+} \geq 0$ for all $x$ by definition, it immediately holds that
\begin{equation*}
    \mathcal{C} = \{x : s(x) \leq 0\} = \{x : [s(x)]_{+} \leq 0\} = \{x : [s(x)]_{+} = 0\}
    	\text{.}
\end{equation*}
By the monotonicity of Lebesgue integration, we obtain that the feasibility set of~\eqref{P:support_sampling_eq} is
\begin{equation}\label{E:support_sampling_feas}
\begin{aligned}
    \mathcal{F} &= \Big\{
    	\mu \in \spprobmeas : \mathbb{E}_{x\sim\mu}\big[ [s(x)]_{+} \big] \leq 0
    \Big\}
    \\
    {}&= \Big\{
    	\mu \in \spprobmeas : \mathbb{E}_{x\sim\mu}\big[ [s(x)]_{+} \big] = 0
    \Big\}
    \\
    {}&= \Big\{
       	\mu \in \spprobmeas : [s(x)]_{+} = 0 \text{,}\ \ \mu\text{-a.e.}
    \Big\}
    	\text{.}
\end{aligned}
\end{equation}
In other words, the feasibility set of~\eqref{P:support_sampling_eq} is in fact~$\mathcal{F} = \{ \mu \in \spprobmeas : \mu(\mathcal{C}) = 1 \} = \calP_2(\calC)$.

\begin{figure}[t]
	\begin{minipage}[t]{0.48\columnwidth}
		\vspace{0pt}
		\centering
		\includegraphics[width=\linewidth]{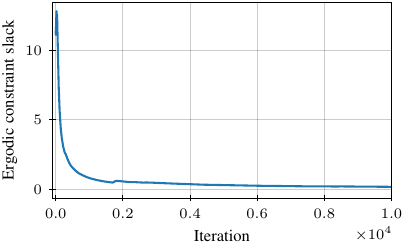}

		{\small (a)}
	\end{minipage}
	\hfill
	\begin{minipage}[t]{0.48\columnwidth}
		\vspace{0pt}
		\centering
		\includegraphics[width=\linewidth]{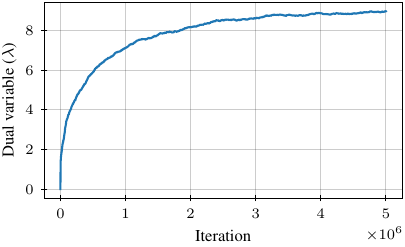}

		{\small (b)}
	\end{minipage}
	\caption{One-dimensional truncated Gaussian sampling: (a)~Ergodic average of the constraint function~(slack) and (b)~Evolution of the dual variable~$\lambda$.}
	\label{F:1d_gaussian_slacks}
\end{figure}

\begin{figure}[t]
	\begin{minipage}[t]{0.48\columnwidth}
		\vspace{0pt}
		\centering
		\includegraphics[width=\linewidth]{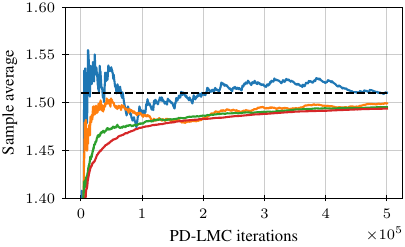}

		{\small (a)}
	\end{minipage}
	\hfill
	\begin{minipage}[t]{0.48\columnwidth}
		\vspace{0pt}
		\centering
		\includegraphics[width=\linewidth]{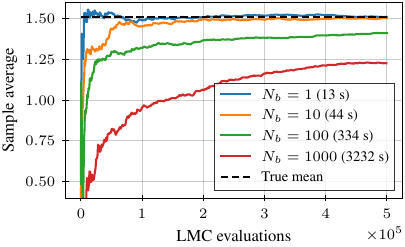}

		{\small (b)}
	\end{minipage}

	\caption{The effect of the mini-batch size~$N_b$ on \pdlmc for sampling from a 1D truncated Gaussian: Estimated mean vs.\ (a)~iteration and (b)~LMC evaluations.}
	\label{F:1d_gaussian_vs_N}
\end{figure}

To illustrate the use of~\eqref{P:support_sampling_eq}, consider the one-dimensional truncated Gaussian sampling problem from~\Cref{S:experiments}. Namely, we wish to sample from a standard Gaussian distribution~$\calN(0,1)$ truncated to~$\calC = [1,3]$. In the language of~\eqref{P:support_sampling_eq}, we take~$f(x) = x^2/2$~(i.e., $\pi \propto e^{-x^2/2}$) and~$s_i(x) = (x-1)(x-3)$. In order to satisfy the assumptions of our convergence guarantees~(particularly~\ref{A:strict_feasibility}), we leave some slack in the constraints by considering~$\E_{\mu} [ [ s_i(x) ]_+ ] \leq 0.005$. This also helps with the numerical stability of the algorithm. Fig.~\ref{F:1d_truncated_gaussian} shows histograms of the samples obtained using~\pdlmc~(\Cref{L:pdlmc} with~$\eta_x = \eta_\lambda = 10^{-3}$), the projected LMC~(Proj.\ LMC, $\eta = 10^{-3}$) from~\cite{Bubeck18s}, and the mirror LMC~($\eta = 10^{-3}$) from~\cite{Ahn20e}. In all cases, we take~$5 \times 10^6$ samples and keep only the second half.

Observe that, due to the projection step, Proj.\ LMC generates an excess of samples close to the boundaries. In fact, it generates over three times more than required. This leads to an underestimation of the distribution mean and variance~(Table~\ref{F:moments}). A similar effect is observed for mirror LMC. In contrast, \pdlmc provides a more accurate estimate. Nevertheless, \pdlmc imposes constraints on the distribution~$\mu$ rather than its samples. Indeed, note from~\eqref{E:support_sampling_feas} that its feasibility set is such that samples belong to~$\calC$ \emph{almost surely}, which still allows for a (potentially infinite) number of realizations outside of~$\calC$. Yet, though \pdlmc is not an \emph{interior-point method}, Theorems~\ref{T:main_convex}--\ref{T:main_lsi} show that excursions of iterates outside of~$\calC$ become less frequent as the algorithm progresses. We can confirm this is indeed the case in Fig.~\ref{F:1d_gaussian_slacks}a, which shows the ergodic average of~$[s(x)]_{+}$ along the samples of \pdlmc. Note that it almost vanishes by iteration~$10^{4}$ even though the dual variable~$\lambda$ only begins to stabilize later~(Fig.~\ref{F:1d_gaussian_slacks}b). This is not surprising given that it is guaranteed by Prop.~\ref{T:ergodic_feasibility}. In fact, only roughly~$2\%$ of the samples displayed in Fig.~\ref{F:1d_truncated_gaussian} are not in~$\calC$.

Before proceeding, we examine whether the convergence of \pdlmc could be improved by averaging more than one LMC samples when updating the dual variables, i.e., using mini-batches in steps~\ref{E:pdlmc_lambda}--\ref{E:pdlmc_nu} of \Cref{L:pdlmc}. Mini-batches will reduce the variance of the dual updates, although at the cost of additional LMC steps per iteration. To compensate for this fact, Fig.~\ref{F:1d_gaussian_vs_N}b displays the evolution of the ergodic average of \pdlmc samples as a function of the number of LMC evaluations rather than the number of iterations~(as in Fig.~\ref{F:1d_gaussian_vs_N}a). Notice that, in this application, increasing the number of LMC samples~$N_b$ does not lead to faster convergence. This illustrates that, though mini-batches could be useful in some applications~(particularly when the constraints are \emph{not} convex, as in Section~\ref{S:lsi_convergence}), it is not immediate that their benefits always outweigh the increased computational cost. Oftentimes, using a single LMC sample is more than enough. It is worth noting that using~\pdlmc with a large mini-batch~$N_b$ was suggested in~\cite{Liu21s} to approximate the expectation needed by their continuous-time algorithm. As we see here, this is neither necessary nor always beneficial.

\begin{figure}[t]
\begin{minipage}[t]{0.47\columnwidth}
\vspace{0pt}
\centering
\includegraphics[width=\linewidth]{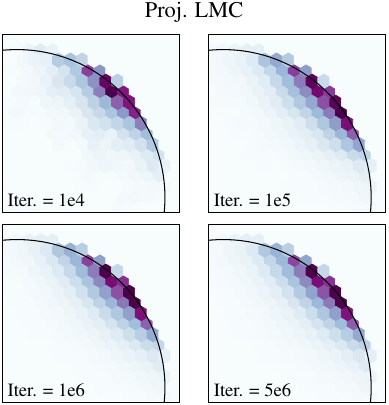}

{\small (a)}
\end{minipage}
\hfill
\begin{minipage}[t]{0.47\columnwidth}
\vspace{0pt}
\centering
\includegraphics[width=\linewidth]{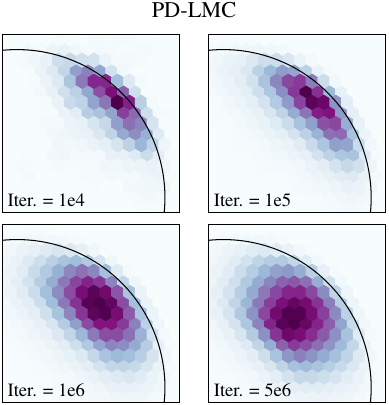}

{\small (b)}
\end{minipage}
\caption{Density estimate of two-dimensional truncated Gaussian using samples from (a)~Proj.~LMC and (b)~\pdlmc.}
	\label{F:2d_gaussian_hist}
\end{figure}

We now turn to a more challenging, two-dimensional applications. We seek to sample from a Gaussian located at~$[2, 2]$ with covariance~$\diag([1,1])$ restricted to an~$\ell_2$-norm unit ball~(Fig.~\ref{F:1d_truncated_gaussian}). Specifically, we use~$f(x) = \norm{x}^2/2$~(i.e., $\pi \propto e^{-\norm{x}^2/2}$) and~$s_i(x) = \norm{x}^2 - 1$. Once again, we leave some slack to the algorithm by taking the constraint in~\eqref{P:support_sampling_eq} to be~$\E_{\mu} [ [ s_i(x) ]_+ ] \leq 0.001$. For reference, we also display samples from the real distribution obtained using rejection sampling.

This is indeed a challenging problem. The boundary of~$\calC$ is 2 standard deviations away from the mean of the target distribution, which means that samples from the target~$\pi$ are extremely scarce this region. Indeed, using the untruncated Gaussian as a proposal for rejection sampling yields an acceptance rate of approximately~$1\%$. The strong push of the potential~$f$ towards the exterior of~$\calC$ leads Proj.~LMC~($\eta = 10^{-3}$; the last~$10^6$ samples are used after running $5 \times 10^6$~iterations) to be now even more concentrated around its boundary. In fact, almost~$25\%$ of its samples are in an annular region of radius~$[0.999,1)$, where only~$0.14\%$ of the samples should be according to rejection sampling. Indeed, note from Fig.~\ref{F:2d_gaussian_hist}a, that even as iterations advance, the samples of Proj.~LMC continue concentrate close to the boundary.

In contrast, \pdlmc~($\eta_x = 10^{-3}$ and~$\eta_\lambda = 2 \times 10^{-1}$) only place~$1.8\%$ of its samples outside of~$\calC$, mostly during the initial phase of the algorithm~(Fig.~\ref{F:2d_gaussian_hist}a). Indeed, the average of the constraint function along samples from \pdlmc essentially vanishes around iteration~$5 \times 10^4$. Achieving this requires larger values of~$\lambda$~(on the order of~$250$, Fig.~\ref{F:2d_gaussian_slacks}b) compared to the one-dimensional case~(Fig.~\ref{F:1d_gaussian_slacks}b). This reflects the difficulty of constraining samples to~$\calC$ in this instance, a statement formalized in the perturbation results of Prop.~\ref{T:strong_duality}(iv). Due to the more amenable numerical properties of the barrier function, mirror LMC~($\eta = 10^{-3}$) performs well without concentrating samples on the boundary~($0.15\%$ of the samples on the annular region of radius~$[0.999,1)$).

\begin{figure}[t]
	\begin{minipage}[t]{0.48\columnwidth}
		\vspace{0pt}
		\centering
		\includegraphics[width=\linewidth]{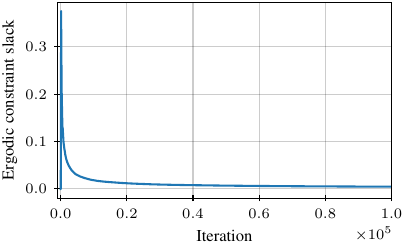}

		{\small (a)}
	\end{minipage}
	\hfill
	\begin{minipage}[t]{0.48\columnwidth}
		\vspace{0pt}
		\centering
		\includegraphics[width=\linewidth]{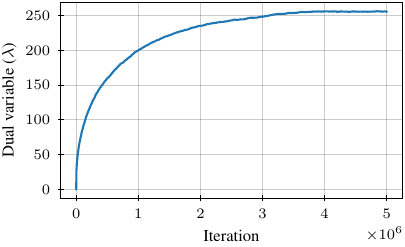}

		{\small (b)}
	\end{minipage}
	\caption{Two-dimensional truncated Gaussian sampling: (a)~Ergodic average of the constraint function~(slack) and (b)~Evolution of the dual variable~$\lambda$.}
	\label{F:2d_gaussian_slacks}
\end{figure}

\newpage

\subsection{Rate-constrained Bayesian models}
\label{X:fairness}

While rate constraints have become popular in ML due to their central role in fairness~(see, e.g., \cite{Kearns18p}), they find applications in robust control~\cite{Schwarm99c, Li00r, Borrelli17p} and to express other requirements on the confusion matrix, such as precision, recall, and false negatives~\cite{Cotter19o}. For illustration, we consider here the problem of fairness in Bayesian classification.

Let~$q(x;\theta) = \Pr[y = 1 \vert x,\theta]$ denote the probability of a positive outcome~($y = 1$) given the observed features~$x \in \calX$ and the parameters~$\theta$ distributed according to the posterior~$\pi$. This posterior is determined, e.g., by some arbitrary Bayesian model based on observations~$\{(x_n,y_n)\}_{n=1,\dots,N}$. Hence, $\E_{\theta \sim \pi} [q(x;\theta)]$ denotes the likelihood of a positive outcome for~$x$. Consider now a protected group, represented by a measurable subset~$\calG \subset \calX$, for which we wish to enforce statistical parity. In other words, we would like the prevalence of positive outcomes to be roughly the same as that of the whole population. Thus, we now want to sample not from the posterior~$\pi$, but from a close-by distribution of parameters~$\theta$ that ensures this parity. Explicitly, for some tolerance~$\delta > 0$, we want to sample from
\begin{prob}\label{P:fairness_average}
	\mu^\star \in \min_{\mu \in \spprobmeas}& &&\KL(\mu \| \pi)
	\\
	\subjectto& &&\E_{x, \theta \sim \mu}
		\!\big[ q(x ; \theta) \mid \calG \big] \geq
	\E_{x, \theta \sim \mu}
		\!\big[ q(x ; \theta) \big] - \delta
	\text{.}
\end{prob}
Naturally, we can account for more than one protected by incorporating additional constraints.

\begin{figure}[t]
\begin{minipage}[t]{0.48\columnwidth}
\vspace{0pt}
\centering
\includegraphics[width=\linewidth]{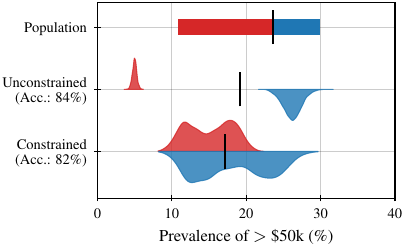}

{\small (a)}
\end{minipage}
\hfill
\begin{minipage}[t]{0.48\columnwidth}
\vspace{0pt}
\centering
\includegraphics[width=\linewidth]{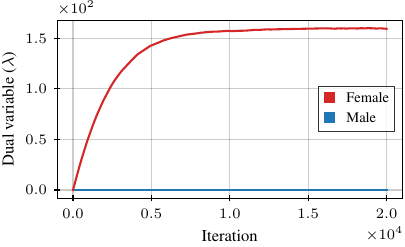}

{\small (b)}
\end{minipage}
\caption{Fair Bayesian logistic regression on the Adult dataset: (a)~prevalence of positive outputs and (b)~dual variables.}
	\label{F:fairness}
\end{figure}

In our experiments, we take~$\pi$ to be the posterior of a Bayesian logistic regression model for the Adult dataset from~\cite{Dua17u}~(details on data pre-processing can be found in~\cite{Chamon20p}). The~$N = 32561$ data points in the training set are composed of~$d = 62$ socio-economical features~($x \in \setR^{d}$, including the intercept) and the goal is to predict whether the individual makes more than~US\$~$50000$ per year~($y \in \{0,1\}$). The posterior is obtained by combining a binomial log-likelihood with independent zero-mean Gaussian (log)priors~($\sigma^2 = 3$) on each parameter of the model, i.e., we consider the potential
\begin{equation}\label{E:fairness_posterior}
    f(\beta) = \sum_{n=1}^N \log (1 + e^{-(2y_n - 1) x_n^\top \theta})
    	+ \sum_{i=0}^d \frac{\theta_i^2}{2 \sigma^2}
    	\text{.}
\end{equation}

We begin by using the LMC algorithm from~\eqref{E:lmc}~(i.e., we impose no constraints) to collect samples of the coefficients~$\theta$ from this posterior~($\eta = 10^{-4}$; the last~$10^4$ samples are used after running $2 \times 10^4$~iterations). We find that, while the probability of positive outputs is~$19.1\%$ across the whole test set, it is~$26.2\%$ among males and~$0.05\%$ among females. Looking at the distribution of this probability over the unconstrained posterior~$\pi$~(Fig.~\ref{F:fairness}a), we see that this behavior goes beyond the mean. The model effectively amplifies the inequality already present in the test set, where the prevalence of positive outputs is~$30.6\%$ among males and~$10.9\%$ among females.

To overcome this disparity, we consider \emph{gender} to be the protected class in~\eqref{P:fairness_average}, constraining both~$\calG_\text{female}$ and~$\calG_\text{male}$. We formulate the constraint of~\eqref{P:fairness_average} using an empirical distribution induced from the data. Explicitly, we consider constraints
\begin{align*}
	\dfrac{1}{\abs{\calG_\text{female}}}
		\sum_{n \in \calG_\text{female}}\E_{\theta \sim \mu}
		\!\big[ q(x_n ; \theta) \big] &\geq
	\dfrac{1}{N}\sum_{n = 1}^N \E_{\theta \sim \mu}
		\!\big[ q(x_n ; \theta) \big] - \delta
	\\
	\dfrac{1}{\abs{\calG_\text{male}}}
		\sum_{n \in \calG_\text{male}}\E_{\theta \sim \mu}
		\!\big[ q(x_n ; \theta) \big] &\geq
	\dfrac{1}{N}\sum_{n = 1}^N \E_{\theta \sim \mu}
		\!\big[ q(x_n ; \theta) \big] - \delta
\end{align*}
where~$\calG_\text{female},\calG_\text{male} \subseteq \{1,\dots,N\}$ partition the data set. For these experiments, we take~$\delta = 0.01$. Using PD-LMC~($\eta_x = 10^{-4}$, $\eta_\lambda = 5 \times 10^{-3}$), we then obtain a new set of samples from the logistic regression parameters~$\theta$ that lead to a prevalence of positive outcomes~(in the test set) of~$17.1\%$ over the whole population, $18.1\%$ for males, and~$15.1\%$ for females. In fact, we notice a substantial overlap between the distributions of this probability across the constrained posterior~$\mu^\star$ for male and female~(Fig.~\ref{F:fairness}a). Additionally, this substantial improvement over the previously observed disparity comes at only a minor reduction in accuracy. Though both distributions change considerably, notice from the value of the dual variables that these changes are completely guided by the \emph{female} group. Indeed, $\lambda_\text{male} = 0$ throughout the execution of~\pdlmc~(Fig.~\ref{F:fairness}b).

Before proceeding, we once again examine the effect of using multiple LMC samples to update the dual variables, i.e., using mini-batches in steps~\ref{E:pdlmc_lambda}--\ref{E:pdlmc_nu} of \Cref{L:pdlmc}. Fig.~\ref{F:fairness_vs_N} shows the distribution of the prevalence of positive predictions~($> \$50$k) for different mini-batch sizes~$N_b$. In all cases, we collect~$2 \times 10^4$ samples, which means that we evaluate~$2N_b\times 10^4$ LMC updates~(step~\ref{E:pdlmc_x} in \Cref{L:pdlmc}). Same as in the 1D truncated Gaussian case, we notice no difference between the resulting distributions. This is to be expected given our results~(Theorem~\ref{T:main_convex}). The computation time, on the other hand, increases considerably with the mini-batch size. Once again, we note that~\pdlmc with a large mini-batch~$N_b$ was used in the experiments of~\cite{Liu21s} to overcome the challenge of computing an expectation in their dual variable updates. In turns out that this computationally intensive modification is not necessary.

\begin{figure}[t]
	\centering
	\includegraphics{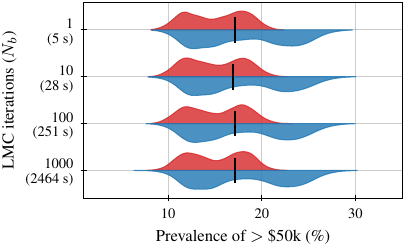}

	\caption{The effect of the mini-batch size~$N_b$ on \pdlmc in fair Bayesian classification.}
	\label{F:fairness_vs_N}
\end{figure}

\subsection{Counterfactual sampling}
\label{X:counterfactual}

Previous applications were primarily interested in \emph{sampling from~$\mu^\star$}, the constrained version of the target distribution~$\pi$. The goal of \emph{counterfactual sampling}, on the other hand, is to \emph{probe} the probabilistic model~$\pi$ by evaluating its compatibility with a set of moment conditions. It is therefore interested not only in~$\mu^\star$, but in how each condition affects the value~$P^\star = \KL(\mu^\star\|\pi)$. We next describe how constrained sampling can be used to tackle this problem.

Let~$\pi$ denote a reference probabilistic model, such as the posterior of the Bayesian logistic model in~\eqref{E:fairness_posterior}. Standard Bayesian hypothesis tests can be used to evaluate the validity of \emph{actual statements} such as ``is it true that~$\E_{x \sim \pi}[g(x)] \leq 0$? or~$\E_{x \sim \pi}[h(x)] = 0$?'' Hence, we could check ``is $\pi$ more likely to yield a positive output for a male than a female individual?''~(from the distributions under \emph{Unconstrained} in Fig.~\ref{F:fairness}a, this is probably the case). In contrast, counterfactual sampling is concerned with \emph{counterfactual statements} such as ``how would the world have been if~$\E[g(x)] \leq 0$?'' In the case of fairness, ``how would the model have been if it predicted positive outcomes more equitably?''

\begin{figure}[t]
\begin{minipage}[t]{0.48\columnwidth}
\vspace{0pt}
\centering
\includegraphics{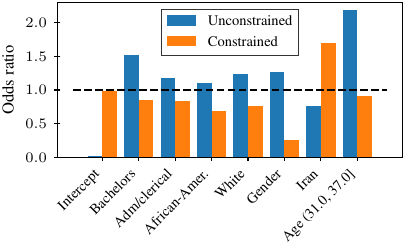}

{\small (a)}
\end{minipage}
\hfill
\begin{minipage}[t]{0.48\columnwidth}
\vspace{0pt}
\centering
\includegraphics[width=\linewidth]{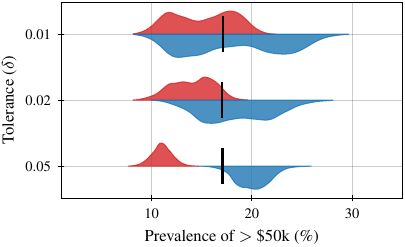}

{\small (b)}
\end{minipage}
\caption{Counterfactual sampling in fair Bayesian logistic regression: (a)~Selected (mean) coefficients and (b)~different tolerances.}
	\label{F:fairness_cf}
\end{figure}

Constrained sampling evaluates these counterfactual statements in two ways. First, by providing realizations of this alternative, counterfactual world~($\mu^\star$). For instance, we can inspect the difference between realizations of~$\pi$, obtained using the traditional LMC in~\eqref{E:lmc}, and~$\mu^\star$, obtained using \pdlmc~(\Cref{L:pdlmc}). In Fig.~\ref{F:fairness_cf}a, we show the mean of some coefficients of the Bayesian logistic models from Section~\ref{X:fairness}. Notice that it is not enough to normalize the Intercept and reduce the advantage given to males~(\emph{Female} is encoded as $\text{Male} = 0$). This alternative model also compensates for other correlated features, such as education~(\emph{Bachelor}), profession~(\emph{Adm/clerical}), and age.

Second, constrained sampling evaluates the ``compatibility'' of each counterfactual condition~(constraint) world with the reference model. While \Cref{L:pdlmc} does not evaluate~$P^\star$ explicitly, it provides measures of its sensitivity to perturbations of the constraints: the Lagrange multipliers~$(\lambda^\star,\nu^\star)$. Indeed, recall from Prop.~\ref{T:strong_duality} that
\begin{equation*}
    \mu^\star \propto \pi \times \left( \prod_{i = 1}^I e^{-\lambda_i^\star g_i} \right)
    	\times \left( \prod_{j = 1}^J e^{- \nu_j^\star h_j} \right)
    	\text{.}
\end{equation*}
Hence, $(\lambda^\star,\nu^\star)$ describe the magnitude of \emph{tilts} needed for the reference model~$\pi$ to satisfy the conditions~$\E[g(x)] \leq 0$ or~$\E[h(x)] = 0$. This relation is made explicit in Prop.~\ref{T:strong_duality}(iv).

Concretely, observe that the dual variable relative to the constraint on the male subgroup is always zero~(Fig.~\ref{F:fairness}b). This implies that~$\pi$ is fully compatible with the condition
\begin{equation*}
    \dfrac{1}{\abs{\calG_\text{male}}}\sum_{n \in \calG_\text{male}}\E_{\theta \sim \mu}
   		\!\big[ q(x_n ; \theta) \big] \geq
   	\dfrac{1}{N}\sum_{n = 1}^N \E_{\theta \sim \mu}
   		\!\big[ q(x_n ; \theta) \big] - \delta
   		\text{,}
\end{equation*}
i.e., the statement ``the model predicts positive outcomes for males on average at least as much as for the whole population.'' In contrast, accommodating statistical parity for females requires considerable deviations from the reference model~$\pi$~($\lambda_\text{female}^\star \approx 160$). Without recalculating~$\mu^\star$, we therefore know that even small changes in the tolerance~$\delta$ for the female constraint would substantially change the distribution of outcomes. This statement is confirmed in Fig.~\ref{F:fairness_cf}b. Notice that this is only possible due to the primal-dual nature of \pdlmc.

\subsubsection{Stock market model}

\begin{figure}[t]
\centering
\includegraphics{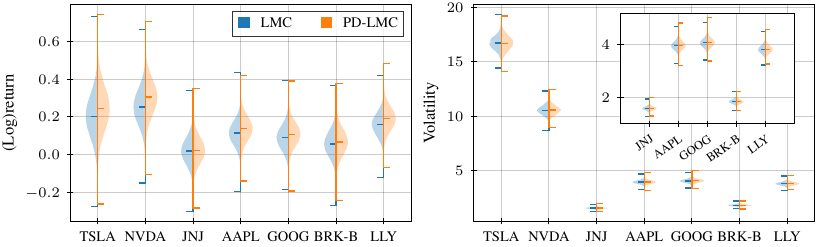}
\caption{Counterfactual sampling of the stock market under a $20\%$ average return increase on each stock: mean~($\rho$) and variance~[$\diag(\Sigma)$] distributions.}
	\label{F:stock_return_violin}
\end{figure}

\begin{figure}[t]

\begin{minipage}[b]{0.48\columnwidth}
	\centering
	\includegraphics[width=\columnwidth]{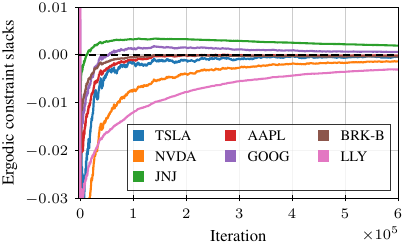}

	\hspace{11mm}{\small (a)}
\end{minipage}
\hfill
\begin{minipage}[b]{0.48\columnwidth}
	\centering
	\includegraphics[width=\columnwidth]{all_stocks_return_nu}

	\hspace{11mm}{\small (b)}
\end{minipage}

\caption{Counterfactual sampling of the stock market under a $20\%$ average return increase on each stock: (a)~ergodic average of constraint functions~(slacks) and (b)~dual variables.}
	\label{F:stock_return_lambda}
\end{figure}

Counterfactual analyses based on the dual variables become more powerful as the number of constraints grow. To see this is the case, consider the \emph{Bayesian stock market} model introduced in \Cref{S:csmp}. Here, $\pi$ denotes the posterior model for the (log-)returns of~$7$ assets~(\texttt{TSLA}, \texttt{NVDA}, \texttt{JNJ}, \texttt{AAPL}, \texttt{GOOG}, \texttt{BRK-B}, and~\texttt{LLY}). The dataset is composed of~$5$ years of adjusted closing prices for a total of~$1260$ points per asset. The posterior is obtained by combining a Gaussian likelihood~$\calN(\rho,\Sigma)$ with Gaussian prior on the mean~$\rho$~[$\calN(0,3 I)$] and an inverse Wishart prior on the covariance~$\Sigma$~(with parameters~$\Psi = I$ and~$\nu = 12$). Using the LMC algorithm~($\eta = 10^{-3}$; the last~$3 \times 10^5$ samples are used after running $6 \times 10^5$~iterations), we collect samples from this posterior and estimate the mean and variance of the (log-)return for each stock~(Table~\ref{F:stocks}). In this case, $\Sigma$ is initialized to~$10 \times I$.

We might now be interested in understanding what the market would look like if all stocks were to incur a~$20\%$ increase in their average (log-)returns. To do so, we use~\pdlmc~($\eta_x = 10^{-3}$ and~$\eta_\nu = 6 \times 10^{-3}$) to solve the following constrained sampling problem
\begin{prob}\label{P:stock_return}
 \minimize_{\mu \in \spprobmeas}& &&\KL(\mu \| \pi)
 \\
 \text{subject to}& &&\E_{(\rho,\Sigma) \sim \mu} \!\big[ \rho_i \big] = 1.2 \bar{\rho}_i
 	\text{,} \quad i = 1,\dots,7
 	\text{,}
\end{prob}
where~$\bar{\rho}_i$ is the mean (log-)return of the $i$-th stock shown in Table~\ref{F:stocks}. The distribution of~$\rho$ and the diagonal of~$\Sigma$ are compared to those from the unconstrained model in Fig.~\ref{F:stock_return_violin}. Notice that, though we only impose constraints on the average returns~$\rho$, we also see small changes in the stock volatilities.

\begin{figure}[t]
\centering
\includegraphics{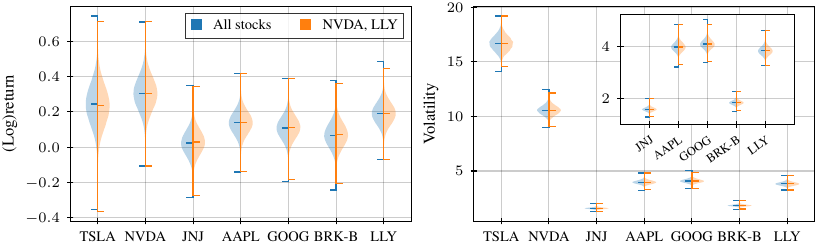}
\caption{Distribution of mean~($\rho$) and variance~[$\diag(\Sigma)$] of the stock market constrained to a $20\%$ average return increase on all stocks vs.\ only \texttt{LLY} and~\texttt{NVDA}.}
	\label{F:stock_return_violin_vs}
\end{figure}

\begin{table}[t]
\centering
\caption{Mean $\pm$ standard deviation of mean (log-)returns~($\rho$).}
\label{F:stocks}
\begin{tabular}{l|c|c|c}
\hline
& Reference model~($\pi$)
& $+20\%$ (all stocks)
& $+20\%$ (\texttt{LLY} and \texttt{NVDA})
\\\hline
\texttt{TSLA}  & $0.20 \pm 0.12$ & $0.24 \pm 0.12$ & $0.23 \pm 0.12$ \\
\texttt{NVDA}  & $0.25 \pm 0.09$ & $0.31 \pm 0.09$ & $0.31 \pm 0.09$ \\
\texttt{JNJ}   & $0.02 \pm 0.07$ & $0.02 \pm 0.07$ & $0.03 \pm 0.07$ \\
\texttt{AAPL}  & $0.11 \pm 0.06$ & $0.14 \pm 0.06$ & $0.14 \pm 0.06$ \\
\texttt{GOOG}  & $0.09 \pm 0.06$ & $0.11 \pm 0.06$ & $0.11 \pm 0.06$ \\
\texttt{BRK-B} & $0.06 \pm 0.07$ & $0.07 \pm 0.07$ & $0.07 \pm 0.07$ \\
\texttt{LLY}   & $0.16 \pm 0.06$ & $0.19 \pm 0.06$ & $0.19 \pm 0.06$ \\\hline
\end{tabular}
\end{table}

Inspecting the dual variables~(Fig.~\ref{F:stock_return_lambda}b), we notice that three dual variables are essentially zero~(\texttt{TSLA}, \texttt{AAPL}, and~\texttt{BRK-B}). This means that their increased returns are completely dictated by those of other stocks. Said differently, their returns increasing~$20\%$ is consistent with the reference model~$\pi$ \emph{conditioned} on the other returns increasing. Proceeding, two stocks have negative dual variables~(\texttt{LLY} and~\texttt{NVDA}). This implies that bringing their constraints \emph{down to~$\bar{\rho}_i$} would yield a \emph{decrease} in~$P^\star$~(distance to the reference model~$\pi$). This is in contrast to~\texttt{JNJ} and~\texttt{GOOG}, whose positive~$\lambda$'s imply that we should should \emph{increase} their returns to reduce~$P^\star$. Indeed, by inspecting the ergodic slacks~(Fig.~\ref{F:stock_return_lambda}a) we see that all stocks approach zero~(i.e., feasibility), but that~\texttt{JNJ} and~\texttt{GOOG} do so from above. This behavior is expected according to Prop.~\ref{T:ergodic_feasibility}.

These observations show two things. First, that an increase in the average returns of~\texttt{LLY} and~\texttt{NVDA} is enough to drive up the returns of all other stocks. In fact, it leads to essentially the same distribution as if we had required the increase to affect all stocks~(Fig.~\ref{F:stock_return_violin_vs}). Second, that the increase we would see in \texttt{JNJ} (and to a lesser extent~\texttt{GOOG}) would actually be larger than~$20\%$. Once again, we reach these conclusion without any additional computation. Their accuracy can be corroborated by the results in Table~\ref{F:stocks}.

%% file: app_gaussian_duality.tex

In this section we illustrate the result in Prop.~\ref{T:strong_duality}, i.e., we show that given solutions~$(\lambda^\star,\nu^\star)$ of~\eqref{P:dual}, the constrained sampling problem~\eqref{P:constrained_sampling} reduces to sampling from~$\mu_{\lambda^\star\nu^\star} \propto e^{-U(\cdot,\lambda^\star,\nu^\star)}$.

Indeed, consider a standard Gaussian target, i.e., $\pi \propto e^{-\|x\|^2/2}$, and the linear moment constraint~$\mathbb{E}[x] = b$, for~$b \in \mathbb{R}^d$. This can be cast as~\eqref{P:constrained_sampling} with~$f(x) = \|x\|^2/2$ and~$h(x) = b - x$~(no inequality constraints, i.e., $I = 0$). Clearly, the solution of~\eqref{P:constrained_sampling} in this case is~$\mu^\star = \mathcal{N}(b,I)$, i.e., a Gaussian distribution with mean~$b$. What Prop~\ref{T:strong_duality} claims is that rather than directly solving~\eqref{P:constrained_sampling}, we can solve~\eqref{P:dual} to obtain a Lagrange multiplier~$\nu^\star$ such that~$\mu^\star = \mu_{\nu^\star}$ for~$\mu_{\nu}$ defined as in~\eqref{E:lagrangian}.

In this setting, we can see this is the case by doing the computations explicitly. Indeed, we have
\begin{equation*}
	\mu_{\nu}(x) \propto
		\pi(x) e^{-\nu^\top h(x)}
		= \exp\bigg[ -\dfrac{\|x\|^2}{2} -\nu^\top (b-x) \bigg].
\end{equation*}
Completing the squares, we then obtain
\begin{equation}\label{E:gaussian_mu_nu}
	\mu_{\nu}(x) \propto \exp\bigg[ -\dfrac{\|x - \nu\|^2}{2}
		+ \dfrac{\|\nu\|^2}{2} - \nu^\top b \bigg].
\end{equation}
To compute the Lagrange multiplier~$\nu^\star$, notice from the definition of the dual function in~\eqref{E:dual_function} that the dual problem~\eqref{P:dual} is in fact a ratio of normalizing factors. Explicitly,
\begin{equation*}
	\nu^\star = \argmax_{\nu \in \setR^d}\ %
		\log\bigg( \dfrac{\int \pi(x)dx}{\int \mu_\nu(x) dx} \bigg)
		= \argmax_{\nu \in \setR^d}\ \log \Bigg[
			\dfrac{\int \exp\Big( -\frac{\|x\|^2}{2} \Big) dx
		}{
			\exp(\|\nu\|^2/2 -\nu^\top b)
			\int \exp\Big( -\frac{\|x - \nu\|^2}{2} \Big) dx
		}
		\Bigg].
\end{equation*}
Immediately, we obtain
\begin{equation}\label{E:gaussian_nu_star}
	\nu^\star = \argmax_{\nu \in \setR^d}\ -\|\nu\|^2/2 +\nu^\top b = b.
\end{equation}
Note that the dual problem is a concave program, as is always the case~\cite{Bertsekas09c}. To conclude, we can combine~\eqref{E:gaussian_mu_nu} and~\eqref{E:gaussian_nu_star} to get
\begin{equation*}
	\mu_{\nu^\star}(x) \Bigm\vert_{\nu^\star = b} \propto
		\exp\bigg( -\dfrac{\|x - b\|^2}{2} - \dfrac{\|b\|^2}{2} \bigg)
	\Rightarrow
	\mu_{\nu^\star} = \mathcal{N}(b,I) = \mu^\star.
\end{equation*}

The main advantage of using Algorithms~\ref{L:pdlmc} and~\ref{L:dlmc} is that we do not need to determine the Lagrange multipliers~$(\lambda^\star,\nu^\star)$ to then sample from~$\mu_{\lambda^\star\nu^\star} = \mu^\star$. Indeed, Theorems~\ref{T:main_convex} and~\ref{T:main_lsi} show that these stochastic primal-dual methods do both things simultaneously, without explicitly evaluating any expectations.

%% file: app_checklist.tex
\begin{enumerate}

\item {\bf Claims}
    \item[] Question: Do the main claims made in the abstract and introduction accurately reflect the paper's contributions and scope?
    \item[] Answer: \answerYes{}
    \item[] Justification: the abstract and introduction list our contributions, i.e. the theoretical and experimental study of Primal-Dual Langevin Monte Carlo.
    \item[] Guidelines:
    \begin{itemize}
        \item The answer NA means that the abstract and introduction do not include the claims made in the paper.
        \item The abstract and/or introduction should clearly state the claims made, including the contributions made in the paper and important assumptions and limitations. A No or NA answer to this question will not be perceived well by the reviewers.
        \item The claims made should match theoretical and experimental results, and reflect how much the results can be expected to generalize to other settings.
        \item It is fine to include aspirational goals as motivation as long as it is clear that these goals are not attained by the paper.
    \end{itemize}

\item {\bf Limitations}
    \item[] Question: Does the paper discuss the limitations of the work performed by the authors?
    \item[] Answer: \answerYes{}
    \item[] Justification: we discuss the limitations of our theoretical results in the conclusion.
    \item[] Guidelines:
    \begin{itemize}
        \item The answer NA means that the paper has no limitation while the answer No means that the paper has limitations, but those are not discussed in the paper.
        \item The authors are encouraged to create a separate "Limitations" section in their paper.
        \item The paper should point out any strong assumptions and how robust the results are to violations of these assumptions (e.g., independence assumptions, noiseless settings, model well-specification, asymptotic approximations only holding locally). The authors should reflect on how these assumptions might be violated in practice and what the implications would be.
        \item The authors should reflect on the scope of the claims made, e.g., if the approach was only tested on a few datasets or with a few runs. In general, empirical results often depend on implicit assumptions, which should be articulated.
        \item The authors should reflect on the factors that influence the performance of the approach. For example, a facial recognition algorithm may perform poorly when image resolution is low or images are taken in low lighting. Or a speech-to-text system might not be used reliably to provide closed captions for online lectures because it fails to handle technical jargon.
        \item The authors should discuss the computational efficiency of the proposed algorithms and how they scale with dataset size.
        \item If applicable, the authors should discuss possible limitations of their approach to address problems of privacy and fairness.
        \item While the authors might fear that complete honesty about limitations might be used by reviewers as grounds for rejection, a worse outcome might be that reviewers discover limitations that aren't acknowledged in the paper. The authors should use their best judgment and recognize that individual actions in favor of transparency play an important role in developing norms that preserve the integrity of the community. Reviewers will be specifically instructed to not penalize honesty concerning limitations.
    \end{itemize}

\item {\bf Theory Assumptions and Proofs}
    \item[] Question: For each theoretical result, does the paper provide the full set of assumptions and a complete (and correct) proof?
    \item[] Answer: \answerYes{}
    \item[] Justification: For each theoretical result, we worked with justified assumptions, making all the dependencies of the problem clear. We provide clear and detailed proofs in the appendix.
    \item[] Guidelines:
    \begin{itemize}
        \item The answer NA means that the paper does not include theoretical results.
        \item All the theorems, formulas, and proofs in the paper should be numbered and cross-referenced.
        \item All assumptions should be clearly stated or referenced in the statement of any theorems.
        \item The proofs can either appear in the main paper or the supplemental material, but if they appear in the supplemental material, the authors are encouraged to provide a short proof sketch to provide intuition.
        \item Inversely, any informal proof provided in the core of the paper should be complemented by formal proofs provided in appendix or supplemental material.
        \item Theorems and Lemmas that the proof relies upon should be properly referenced.
    \end{itemize}

    \item {\bf Experimental Result Reproducibility}
    \item[] Question: Does the paper fully disclose all the information needed to reproduce the main experimental results of the paper to the extent that it affects the main claims and/or conclusions of the paper (regardless of whether the code and data are provided or not)?
    \item[] Answer: \answerYes{}
    \item[] Justification: For each experimental result, we provide the detailed setting.
    \item[] Guidelines:
    \begin{itemize}
        \item The answer NA means that the paper does not include experiments.
        \item If the paper includes experiments, a No answer to this question will not be perceived well by the reviewers: Making the paper reproducible is important, regardless of whether the code and data are provided or not.
        \item If the contribution is a dataset and/or model, the authors should describe the steps taken to make their results reproducible or verifiable.
        \item Depending on the contribution, reproducibility can be accomplished in various ways. For example, if the contribution is a novel architecture, describing the architecture fully might suffice, or if the contribution is a specific model and empirical evaluation, it may be necessary to either make it possible for others to replicate the model with the same dataset, or provide access to the model. In general. releasing code and data is often one good way to accomplish this, but reproducibility can also be provided via detailed instructions for how to replicate the results, access to a hosted model (e.g., in the case of a large language model), releasing of a model checkpoint, or other means that are appropriate to the research performed.
        \item While NeurIPS does not require releasing code, the conference does require all submissions to provide some reasonable avenue for reproducibility, which may depend on the nature of the contribution. For example
        \begin{enumerate}
            \item If the contribution is primarily a new algorithm, the paper should make it clear how to reproduce that algorithm.
            \item If the contribution is primarily a new model architecture, the paper should describe the architecture clearly and fully.
            \item If the contribution is a new model (e.g., a large language model), then there should either be a way to access this model for reproducing the results or a way to reproduce the model (e.g., with an open-source dataset or instructions for how to construct the dataset).
            \item We recognize that reproducibility may be tricky in some cases, in which case authors are welcome to describe the particular way they provide for reproducibility. In the case of closed-source models, it may be that access to the model is limited in some way (e.g., to registered users), but it should be possible for other researchers to have some path to reproducing or verifying the results.
        \end{enumerate}
    \end{itemize}

\item {\bf Open access to data and code}
    \item[] Question: Does the paper provide open access to the data and code, with sufficient instructions to faithfully reproduce the main experimental results, as described in supplemental material?
    \item[] Answer: \answerYes{}
    \item[] Justification: We will provide a link to a public github repository with a Python code. 
    \item[] Guidelines:
    \begin{itemize}
        \item The answer NA means that paper does not include experiments requiring code.
        \item Please see the NeurIPS code and data submission guidelines (\url{https://nips.cc/public/guides/CodeSubmissionPolicy}) for more details.
        \item While we encourage the release of code and data, we understand that this might not be possible, so “No” is an acceptable answer. Papers cannot be rejected simply for not including code, unless this is central to the contribution (e.g., for a new open-source benchmark).
        \item The instructions should contain the exact command and environment needed to run to reproduce the results. See the NeurIPS code and data submission guidelines (\url{https://nips.cc/public/guides/CodeSubmissionPolicy}) for more details.
        \item The authors should provide instructions on data access and preparation, including how to access the raw data, preprocessed data, intermediate data, and generated data, etc.
        \item The authors should provide scripts to reproduce all experimental results for the new proposed method and baselines. If only a subset of experiments are reproducible, they should state which ones are omitted from the script and why.
        \item At submission time, to preserve anonymity, the authors should release anonymized versions (if applicable).
        \item Providing as much information as possible in supplemental material (appended to the paper) is recommended, but including URLs to data and code is permitted.
    \end{itemize}

\item {\bf Experimental Setting/Details}
    \item[] Question: Does the paper specify all the training and test details (e.g., data splits, hyperparameters, how they were chosen, type of optimizer, etc.) necessary to understand the results?
    \item[] Answer: \answerYes{}
    \item[] Justification: the algorithms and settings of the experiments are rather simple, and detailed in the submission. 
    \item[] Guidelines:
    \begin{itemize}
        \item The answer NA means that the paper does not include experiments.
        \item The experimental setting should be presented in the core of the paper to a level of detail that is necessary to appreciate the results and make sense of them.
        \item The full details can be provided either with the code, in appendix, or as supplemental material.
    \end{itemize}

\item {\bf Experiment Statistical Significance}
    \item[] Question: Does the paper report error bars suitably and correctly defined or other appropriate information about the statistical significance of the experiments?
    \item[] Answer: \answerYes{} 
    \item[] Justification: We give all precisions relative to the the evaluation of our models, either on simulated data or real data, including precising train/set proportions for instance.
    \item[] Guidelines:
    \begin{itemize}
        \item The answer NA means that the paper does not include experiments.
        \item The authors should answer "Yes" if the results are accompanied by error bars, confidence intervals, or statistical significance tests, at least for the experiments that support the main claims of the paper.
        \item The factors of variability that the error bars are capturing should be clearly stated (for example, train/test split, initialization, random drawing of some parameter, or overall run with given experimental conditions).
        \item The method for calculating the error bars should be explained (closed form formula, call to a library function, bootstrap, etc.)
        \item The assumptions made should be given (e.g., Normally distributed errors).
        \item It should be clear whether the error bar is the standard deviation or the standard error of the mean.
        \item It is OK to report 1-sigma error bars, but one should state it. The authors should preferably report a 2-sigma error bar than state that they have a 96\% CI, if the hypothesis of Normality of errors is not verified.
        \item For asymmetric distributions, the authors should be careful not to show in tables or figures symmetric error bars that would yield results that are out of range (e.g. negative error rates).
        \item If error bars are reported in tables or plots, The authors should explain in the text how they were calculated and reference the corresponding figures or tables in the text.
    \end{itemize}

\item {\bf Experiments Compute Resources}
    \item[] Question: For each experiment, does the paper provide sufficient information on the computer resources (type of compute workers, memory, time of execution) needed to reproduce the experiments?
    \item[] Answer: \answerYes{}
    \item[] Justification: Our experiments run on a standard laptop in a few minutes, and are illustrative.
    \item[] Guidelines:
    \begin{itemize}
        \item The answer NA means that the paper does not include experiments.
        \item The paper should indicate the type of compute workers CPU or GPU, internal cluster, or cloud provider, including relevant memory and storage.
        \item The paper should provide the amount of compute required for each of the individual experimental runs as well as estimate the total compute.
        \item The paper should disclose whether the full research project required more compute than the experiments reported in the paper (e.g., preliminary or failed experiments that didn't make it into the paper).
    \end{itemize}

\item {\bf Code Of Ethics}
    \item[] Question: Does the research conducted in the paper conform, in every respect, with the NeurIPS Code of Ethics \url{https://neurips.cc/public/EthicsGuidelines}?
    \item[] Answer: \answerYes{}
    \item[] Justification: In our opinion, this paper does not address societal impact directly, and consider the generic problem of optimization over measures and sampling.
    \item[] Guidelines:
    \begin{itemize}
        \item The answer NA means that the authors have not reviewed the NeurIPS Code of Ethics.
        \item If the authors answer No, they should explain the special circumstances that require a deviation from the Code of Ethics.
        \item The authors should make sure to preserve anonymity (e.g., if there is a special consideration due to laws or regulations in their jurisdiction).
    \end{itemize}

\item {\bf Broader Impacts}
    \item[] Question: Does the paper discuss both potential positive societal impacts and negative societal impacts of the work performed?
    \item[] Answer:  \answerNA{}
    \item[] Justification: In our opinion the paper does not have direct positive or negative social impact. 
    \item[] Guidelines:
    \begin{itemize}
        \item The answer NA means that there is no societal impact of the work performed.
        \item If the authors answer NA or No, they should explain why their work has no societal impact or why the paper does not address societal impact.
        \item Examples of negative societal impacts include potential malicious or unintended uses (e.g., disinformation, generating fake profiles, surveillance), fairness considerations (e.g., deployment of technologies that could make decisions that unfairly impact specific groups), privacy considerations, and security considerations.
        \item The conference expects that many papers will be foundational research and not tied to particular applications, let alone deployments. However, if there is a direct path to any negative applications, the authors should point it out. For example, it is legitimate to point out that an improvement in the quality of generative models could be used to generate deepfakes for disinformation. On the other hand, it is not needed to point out that a generic algorithm for optimizing neural networks could enable people to train models that generate Deepfakes faster.
        \item The authors should consider possible harms that could arise when the technology is being used as intended and functioning correctly, harms that could arise when the technology is being used as intended but gives incorrect results, and harms following from (intentional or unintentional) misuse of the technology.
        \item If there are negative societal impacts, the authors could also discuss possible mitigation strategies (e.g., gated release of models, providing defenses in addition to attacks, mechanisms for monitoring misuse, mechanisms to monitor how a system learns from feedback over time, improving the efficiency and accessibility of ML).
    \end{itemize}

\item {\bf Safeguards}
    \item[] Question: Does the paper describe safeguards that have been put in place for responsible release of data or models that have a high risk for misuse (e.g., pretrained language models, image generators, or scraped datasets)?
    \item[] Answer:  \answerNA{}
    \item[] Justification: Our paper does not present such risks. 
    \item[] Guidelines:
    \begin{itemize}
        \item The answer NA means that the paper poses no such risks.
        \item Released models that have a high risk for misuse or dual-use should be released with necessary safeguards to allow for controlled use of the model, for example by requiring that users adhere to usage guidelines or restrictions to access the model or implementing safety filters.
        \item Datasets that have been scraped from the Internet could pose safety risks. The authors should describe how they avoided releasing unsafe images.
        \item We recognize that providing effective safeguards is challenging, and many papers do not require this, but we encourage authors to take this into account and make a best faith effort.
    \end{itemize}

\item {\bf Licenses for existing assets}
    \item[] Question: Are the creators or original owners of assets (e.g., code, data, models), used in the paper, properly credited and are the license and terms of use explicitly mentioned and properly respected?
    \item[] Answer: \answerNA{}
    \item[] Justification: our paper does not use existing assets
    \item[] Guidelines:
    \begin{itemize}
        \item The answer NA means that the paper does not use existing assets.
        \item The authors should cite the original paper that produced the code package or dataset.
        \item The authors should state which version of the asset is used and, if possible, include a URL.
        \item The name of the license (e.g., CC-BY 4.0) should be included for each asset.
        \item For scraped data from a particular source (e.g., website), the copyright and terms of service of that source should be provided.
        \item If assets are released, the license, copyright information, and terms of use in the package should be provided. For popular datasets, \url{paperswithcode.com/datasets} has curated licenses for some datasets. Their licensing guide can help determine the license of a dataset.
        \item For existing datasets that are re-packaged, both the original license and the license of the derived asset (if it has changed) should be provided.
        \item If this information is not available online, the authors are encouraged to reach out to the asset's creators.
    \end{itemize}

\item {\bf New Assets}
    \item[] Question: Are new assets introduced in the paper well documented and is the documentation provided alongside the assets?
    \item[] Answer: \answerYes{}
    \item[]  Justification: we only have experiments in Python that will be made public. 
    \item[] Guidelines:
    \begin{itemize}
        \item The answer NA means that the paper does not release new assets.
        \item Researchers should communicate the details of the dataset/code/model as part of their submissions via structured templates. This includes details about training, license, limitations, etc.
        \item The paper should discuss whether and how consent was obtained from people whose asset is used.
        \item At submission time, remember to anonymize your assets (if applicable). You can either create an anonymized URL or include an anonymized zip file.
    \end{itemize}

\item {\bf Crowdsourcing and Research with Human Subjects}
    \item[] Question: For crowdsourcing experiments and research with human subjects, does the paper include the full text of instructions given to participants and screenshots, if applicable, as well as details about compensation (if any)?
    \item[] Answer: \answerNA{}
    \item[] Justification: our experiments do not involve crowdsourcing. 
    \item[] Guidelines:
    \begin{itemize}
        \item The answer NA means that the paper does not involve crowdsourcing nor research with human subjects.
        \item Including this information in the supplemental material is fine, but if the main contribution of the paper involves human subjects, then as much detail as possible should be included in the main paper.
        \item According to the NeurIPS Code of Ethics, workers involved in data collection, curation, or other labor should be paid at least the minimum wage in the country of the data collector.
    \end{itemize}

\item {\bf Institutional Review Board (IRB) Approvals or Equivalent for Research with Human Subjects}
    \item[] Question: Does the paper describe potential risks incurred by study participants, whether such risks were disclosed to the subjects, and whether Institutional Review Board (IRB) approvals (or an equivalent approval/review based on the requirements of your country or institution) were obtained?
    \item[] Answer:  \answerNA{}
    \item[] Justification: our study do not involve risk for participants. 
    \item[] Guidelines:
    \begin{itemize}
        \item The answer NA means that the paper does not involve crowdsourcing nor research with human subjects.
        \item Depending on the country in which research is conducted, IRB approval (or equivalent) may be required for any human subjects research. If you obtained IRB approval, you should clearly state this in the paper.
        \item We recognize that the procedures for this may vary significantly between institutions and locations, and we expect authors to adhere to the NeurIPS Code of Ethics and the guidelines for their institution.
        \item For initial submissions, do not include any information that would break anonymity (if applicable), such as the institution conducting the review.
    \end{itemize}

\end{enumerate}